%% file: main/main.tex
\documentclass[onefignum,onetabnum]{siam_style_file/siamart251216}
\usepackage{tikz}
\usepackage{enumitem}
\usepackage{algorithm}
\usepackage{algpseudocode}
\usepackage{amsmath}
\usepackage{rotating}
\usepackage{multirow}

\usepackage{amsfonts}
\usepackage{graphics}
\usepackage[
 letterpaper, top=1.2in, bottom=1.2in, left=1.2in, right=1.2in]{geometry}
\usepackage{xcolor}
\usepackage{amssymb}
\usepackage{psfrag}
\usepackage{graphicx}
\usepackage[utf8]{inputenc}
\usepackage{bm}
\usepackage{mathtools}
\usepackage{epstopdf}
\usepackage{setspace}
\usepackage{amsmath,amssymb,bbm}
\usepackage[utf8]{inputenc}
\usepackage{url}
\usepackage{color}
\usepackage{booktabs}
\newtheorem{assumption}{Assumption}[section]

\input{share_information/ex_shared}

\usepackage{bbm}

\allowdisplaybreaks

\begin{document}

\maketitle


\begin{abstract}
Turbulent dynamical systems are characterized by nonlinear interactions and stochastic effects that generate coupled statistical quantities, such as non-zero higher-order moments, which are difficult to capture from data with accuracy. We propose a two-stage data-driven modeling framework that combines symbolic regression with generative models to jointly identify the governing dynamics and predict their key statistical quantities. In Stage I of the framework, the Finite Expression Method (FEX) is adopted to discover closed-form expressions of the deterministic dynamics, recovering nonlinear interaction terms and external forcing without predefined libraries. In Stage II, generative models are introduced to learn the residual stochastic components as a refined correction to the model error from the Stage I approximation, enabling accurate characterization of higher-order statistics.  Theoretical analysis establishes the consistency of the symbolic estimator and quantifies the estimation error in terms of data size and numerical discretization. 
The model performance is verified through detailed numerical experiments on the stochastic triad  models across multiple regimes, demonstrating that the framework successfully recovers interaction terms and forcing expressions, and accurately predicts statistical moments up to order five. These results highlight the potential of integrating interpretable symbolic discovery with data-driven stochastic modeling for complex turbulent systems.
\end{abstract}

\begin{keywords}
Symbolic discovery, turbulent dynamical systems, generative models, scientific machine learning
\end{keywords}

\section{Introduction}
\label{Introduction}

Turbulent dynamical systems arise in a wide range of scientific and engineering applications and are characterized by strong nonlinear interactions and intrinsic instabilities in a high-dimensional phase space~\cite{majda2006nonlinear,nicholson1983introduction,vallis2017atmospheric}. The presence of a strange attractor and multiple positive Lyapunov exponents leads to rapid amplification of small perturbations, making long-term prediction inherently uncertain and necessitating a stochastic description for the state evolution~\cite{majda2018strategies,qi2017strategies,qi2016low,wang2025simulating}. Despite substantial progress, modeling such systems remains challenging due to imperfect model errors, unresolved scales, and limited observational data~\cite{lorenz1969predictability,majda2010quantifying}. In such turbulent regimes, external forcing injects variability with uncertainty, while nonlinear interactions redistribute this variability across modes and couple statistical quantities of different orders. Consequently, higher-order moments contain essential information beyond low-order summaries, and capturing the underlying forcing and nonlinear interactions is essential for reproducing high-order moment evolutions, which characterize turbulent behaviors.
While recent data-driven methods have shown success in learning local or short-term dynamics, they often rely on low-order statistical representations, such as the mean and variance, which may be insufficient to capture these essential behaviors of turbulent dynamics~\cite{duraisamy2019turbulence,ling2016reynolds}.

In this context, symbolic learning provides an interpretable framework for modeling dynamical systems by identifying a closed-form representation of the underlying dynamics, including nonlinear interaction terms and external forcing. Rather than fitting trajectories or low-order statistics alone, symbolic approaches aim to uncover these components, which directly govern the evolution of high-order statistical quantities. 
A representative approach is Sparse Identification of Nonlinear Dynamical Systems (SINDy)~\cite{brunton2016discovering,lenfesty2025uncovering,messenger2021weak,wanner2024higher}, which formulates model discovery as a sparse regression problem over a predefined library of candidate functions, reducing the task to selecting a small subset of active terms. 
Several variants have been proposed to improve SINDy's performance under different settings. In particular, Weak SINDy~\cite{messenger2021weak} reformulates the regression problem in a weak integral form, avoiding explicit differentiation and improving robustness to noise. However, both SINDy and Weak SINDy rely on a predefined function library, which limits their ability to represent complex forcing expressions that fall outside the chosen basis. Moreover, even when the correct interaction terms are identified, the strong stochasticity in turbulent systems makes accurate coefficient estimation challenging, often leading to violations of underlying conservation laws, such as energy-preserving exchanges among modes.
In contrast, the Finite Expression Method (FEX), as another symbolic approach, removes the restrictions imposed by predefined libraries and represents equations as compositional structures built from a finite set of elementary operators~\cite{du2025learning,hardwick2025solving,lai2025h,liang5327407identifying,liang2025finite,song2025finite}. Specifically, FEX represents candidate expressions as binary trees, enabling hierarchical compositions that capture nested functional forms beyond the linear span of predefined basis functions. This leads to a combinatorial search over closed-form expressions, which is addressed using reinforcement learning within a mixed discrete–continuous optimization framework. As a result, FEX can identify compact and interpretable representations of the underlying dynamics, particularly for nonlinear interactions and external forcing in turbulent systems, while preserving interaction structures consistent with underlying conservation laws.

In this work, we introduce the FEX framework to derive the statistical solution of turbulent dynamical systems, with an emphasis on recovering not only low-order moments but also high-order statistical quantities induced by interaction terms and external forcing. Since accurately reproducing high-order moments requires capturing both deterministic structure and stochastic effects, we propose a two-stage framework in which FEX identifies the deterministic component of the dynamics, including interaction terms and deterministic forcing in Stage I, while an additional Stage II is introduced to model the stochastic effects such as model uncertainty and model errors via generative models, i.e., probabilistic models designed to learn and sample from underlying turbulent dynamics distributions~\cite{goodfellow2014generative,huynh2025score,kingma2013auto,liu2025training,tzen2019neural,wang2026error,xu2024modeling,yang2025generative}, to enable faithful recovery of high-order moments involving highly non-Gaussian distributions. This two-stage framework leverages the complementary strengths of distinct modeling paradigms. 
Stage I uses FEX to recover each term in the governing turbulent dynamics in closed form, avoiding spurious expressions introduced by noise. Stage II then learns the probability distribution evolution of the residual between the FEX-derived deterministic dynamics and the observed dynamics. Although FEX may not perfectly capture the true dynamics, discrepancies in the learned coefficients are effectively corrected through this residual modeling.  See Figure~\ref{fig:two-stage} for an illustration of the framework.

\begin{figure}[!htbp]
    \centering
    \includegraphics[width=1.0\linewidth]{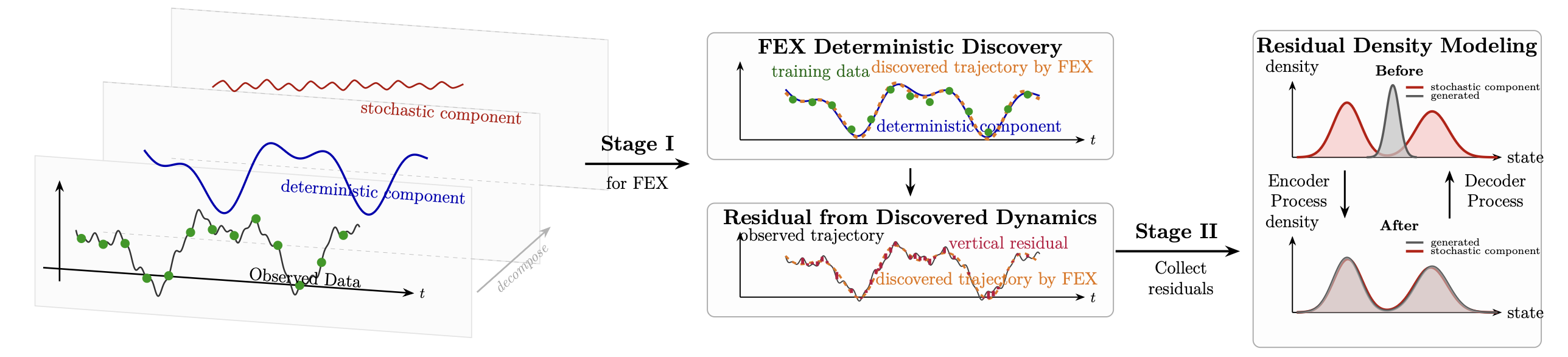}
    \caption{Workflow of two-stage framework: Stage I models determinstic component and Stage II models stochastic component.}
    \label{fig:two-stage}
\end{figure}

Our main contributions in this paper are as follows:
\begin{itemize}
    \item \textbf{Accurate symbolic discovery with limited data across varying turbulence regimes.} Across regimes ranging from deterministic to strongly stochastic, FEX maintains uniformly high skill in recovering the correct deterministic expression formula and accurately identifies nonlinear interaction coefficients, even with limited samples, with performance improving as data increases. The learned coefficients also respect intrinsic physical constraints, such as energy-conserving transfer across modes, through the FEX fine-tuning process.
    \item \textbf{Out-of-training-domain generalization under constant noise.} Under a constant noise assumption, Stage II learns the distribution of residuals between the synthetic data and the corresponding outputs of the Stage I model, enabling trajectory resampling beyond the training regime and improving the long-term estimation of high-order moments.
\end{itemize}

The remainder of this article is organized as follows. Section~\ref{PS} presents the problem formulation. Section~\ref{ASD_Method} describes the FEX framework together with the generative models used in Stage II. Section~\ref{ASD_Implementation} details the implementation of the proposed two-stage framework, while Section~\ref{Numerical_Analysis} provides a theoretical foundation. Section~\ref{Numerical_Experiment} illustrates its numerical performance for high moments recovery, followed by a summarizing discussion in Section~\ref{sec:sum}.

\section{Problem Description}\label{PS}

We consider a $d$-dimensional dynamical system in which a dominant deterministic structure is coupled with stochastic variability. The evolution of the system can be written as
\begin{equation}\label{GeneralTurbul}
    \frac{d\mathbf{u}(t)}{dt}
    = (\mathbf{L} + \mathbf{D})\,\mathbf{u}(t)
    + \mathbf{B}(\mathbf{u}(t), \mathbf{u}(t))
    + \mathbf{F}(t)
    + \boldsymbol{\sigma}(t)\,\dot{\mathbf{W}}_t, 
\end{equation}
where \(\mathbf{u}(t)\in\mathbb{R}^d\)
 denotes the system state at time \(t\). 
The matrix $\mathbf{L}\in\mathbb{R}^{d\times d}$ is skew-symmetric and conserves energy in the linear dynamics, 
while $\mathbf{D}\in\mathbb{R}^{d\times d}$ is negative definite and induces linear dissipation. 
The quadratic term $\mathbf{B}(\mathbf{u},\mathbf{u})\in\mathbb{R}^{d}$ represents nonlinear interactions that preserve energy. 
The forcing term $\mathbf{F}(t)$ and the stochastic term $\boldsymbol{\sigma}(t)\dot{\mathbf{W}}_t$ inject and redistribute energy across modes.
Here, 
$\boldsymbol{\sigma}(t)\in\mathbb{R}^{d\times s}$ denotes the diffusion coefficient matrix, and $\mathbf{W}_t$ is a $s$-dimensional standard Brownian motion, where $s \le d$ allows for possibly low-rank noise. Many complex turbulent dynamical systems can be cast into the abstract form in~\eqref{GeneralTurbul}, inheriting the underlying energy-conserving structure of the linear and nonlinear interactions. Representative examples include the (truncated) Navier--Stokes equations~\cite{pope2001turbulent}, as well as fundamental geophysical models for the atmosphere, ocean, and climate systems that incorporate effects such as rotation, stratification, and topography~\cite{majda2006nonlinear,majda2016introduction,rangan2009multiscale}.

Using the Euler–Maruyama scheme \cite{gu2023stationary}, the dynamics \eqref{GeneralTurbul} can be solved in the discrete form as
\begin{equation}\label{FM1}
\mathbf{u}(t+h)= \mathbf{u}(t) 
+ 
h{\big[(\mathbf{L}+\mathbf{D})\mathbf{u}(t)
+
\mathbf{B}(\mathbf{u}(t),\mathbf{u}(t))
+
\mathbf{F}(t)\big]}
+
{\boldsymbol{\sigma}(t)\big(\mathbf{W}_{t+h}-\mathbf{W}_t\big)}, 
\end{equation}
where $h>0$ is a fixed step size and the Brownian increment satisfies
\(
\mathbf{W}_{t+h}-\mathbf{W}_t \sim \mathcal{N}(0, h\mathbf{I}_d) 
\), with $\mathbf{I}_d$ denoting the $d$-dimensional identity matrix. This representation naturally separates the evolution into a drift term and a stochastic increment, motivating a decomposition into deterministic and stochastic components.
Accordingly, we denote the deterministic component as 
\[
\mathbf{C}(\mathbf{u},t) 
:= 
(\mathbf{L}+\mathbf{D})\mathbf{u}(t)
+
\mathbf{B}(\mathbf{u}(t),\mathbf{u}(t))
+
\mathbf{F}(t),
\]
and the stochastic component as
\[
\mathbf{S}(t,{z};h)
:=
\boldsymbol{\sigma}(t)\big(\mathbf{W}_{t+h}-\mathbf{W}_t\big)
=
\boldsymbol{\sigma}(t)\sqrt{h}\,{z},
\]
where 
 \(
 \mathbf{W}_{t+h}-\mathbf{W}_t = \sqrt{h}\,{z}, 
 \text{ and } {z} \sim \mathcal{N}(0,\mathbf{I}_d).
 \)
Based on this formulation, we employ a two-stage framework as follows: Stage I uses FEX to identify a symbolic representation of $\mathbf{C}(\mathbf{u},t)$ , capturing the structured dynamics, while Stage II uses generative models to both learn $\mathbf{S}(t,{z};h)$ and correct the model error from Stage I, accounting for the residual stochastic variability. Together, these components enable accurate recovery of higher-order statistical behavior in turbulent systems~\eqref{GeneralTurbul}.

In particular, we consider three representative generative models to model the stochastic component: the Training-Free Diffusion Model (TFDM)~\cite{huynh2025score,liu2025training,wang2026error,yang2025generative}, the Stochastic Residual Activation Network (SRAN)~\cite{goodfellow2014generative,tzen2019neural}, and the Variational Autoencoder (VAE)~\cite{kingma2013auto,xu2024modeling}.
TFDM formulates the stochastic process via forward and reverse stochastic differential equations (SDE), where the score function is approximated numerically without explicit training. SRAN provides a lightweight data-driven mapping from Gaussian noise to stochastic residuals. VAE adopts a latent-variable framework that captures stochastic structure through probabilistic encoding and decoding, enabling generation via a learned latent distribution.

For the simplicity of notation, we assume that the observation dataset of turbulent dynamics \eqref{GeneralTurbul} is generated from the discrete formulation \eqref{FM1} and consists of \(M\) independent solution trajectories sampled at uniform time steps over the interval \([0,T]\). Let \(N = T/h\) be the total number of time steps, and define the discrete observation times \(t_n = nh\) for \(n=0,1,\cdots,N.\)
For the \(m\)-th trajectory, the observed state at time \(t_n\) is given by
\[\mathbf{u}^{(m)}(t_n) = \big(u_1^{(m)}(t_n),\cdots,u_{d}^{(m)}(t_n)\big)^\top\in \mathbb{R}^d,\]
where $u_i^{(m)}(t_n)\in\mathbb R$ ($i=1,\cdots d, m=1, \cdots M$) denotes the \(i\)-th component of the state vector. These component-wise observations are further used to define scalar input–output data pairs
\[\big(u_{i}^{(m)}(t_n),u_{i}^{(m)}(t_{n+1})\big) \in \mathbb{R} \times \mathbb{R}, \ \text{for }n=0,\cdots,N-1,\]
which serve as the basic training samples in the subsequent procedure. Finally, collecting the states of all trajectories at a fixed time \(t_n\) yields the following data matrix:
\[\mathbf{u}(t_n) = \big(\mathbf{u}^{(1)}(t_n), \cdots,\mathbf{u}^{(M)}(t_n)\big)^\top \in \mathbb{R}^{M\times d},\]
whose rows correspond to different trajectories and whose columns correspond to different state components.

\section{A Two-Stage Methodology for Statistical Prediction}\label{ASD_Method}
In this section, we introduce the details of FEX and three representative generative models. Due to the nonlinearity and complex attractor structure of turbulent systems, directly learning the full dynamics is computationally challenging. To improve efficiency, we perform symbolic discovery component-wise: for a
$d$-dimensional system, FEX identifies each component of the deterministic vector field separately, and the resulting expressions are combined to form the recovered system. The residual between this symbolic representation and the true dynamics is then passed to Stage II, where TFDM, SRAN, and VAE are used to model the remaining stochastic effects and correct the Stage I model error, enabling accurate recovery of the full dynamics.

\subsection{Finite Expression Method for Deterministic Dynamics}\label{FEX}

In FEX, every symbolic expression  is constructed through a binary tree $\mathcal{T}$, where each node assigns an operator drawn from either the unary operator set $\mathbb{U}$ or the binary operator set $\mathbb{B}$. Typical choices include:
\[
\begin{aligned}
\mathbb{U} &= \left\{\sin(\cdot),\exp(\cdot),\log(\cdot),\text{Id},(\cdot)^2,\frac{\partial}{\partial x_i}(\cdot),\ldots\right\}, &
\mathbb{B} = \left\{+,-,\times,\div,\otimes,\oplus, \ldots\right\},
\end{aligned}
\]
with \(\times\) and \(\div\) denoting element-wise multiplication and division, respectively. \(\otimes\) denotes the tensor product, and \(\oplus\) denotes the direct sum. Here, all unary operators act element-wise on the input and are typically implemented
with trainable scaling and bias, while the binary operators combine two inputs through algebraic operations. For example, in a depth-1 tree \(\mathcal{T}\) as illustrated in Figure~\ref{fig:FEX_tree}, the output\(o_1\) of a unary operator \(u_1\) acting input \(i_0\) is
\(o_1=\alpha \cdot u_1(i_0) + \beta\), with scaling \(\alpha\) and bias \(\beta\); the output of a binary operator \(b_1\) acting on two inputs \(i_{01}, i_{02}\) is \(o_1 = b_1(i_{01}, i_{02})\), with no additional scaling. For deeper trees, this extends recursively, with each operator applied to the outputs of its children before passing the result to its parent, enabling increasingly complex compositions of unary and binary operators evaluated bottom-up from the leaves to the root. In this work, we adopt a depth \(L=3\) tree as the one-dimensional tree block to construct the overall tree in Section~\ref{FEX_search}. Previous experiments~\cite{lai2025h, liang2025finite, song2025finite} indicate that FEX is not sensitive to the choice of the one-dimensional tree block, and that it enables simple and computationally efficient implementations.

\begin{figure}[!htbp]
    \centering
    \includegraphics[width=1\linewidth]{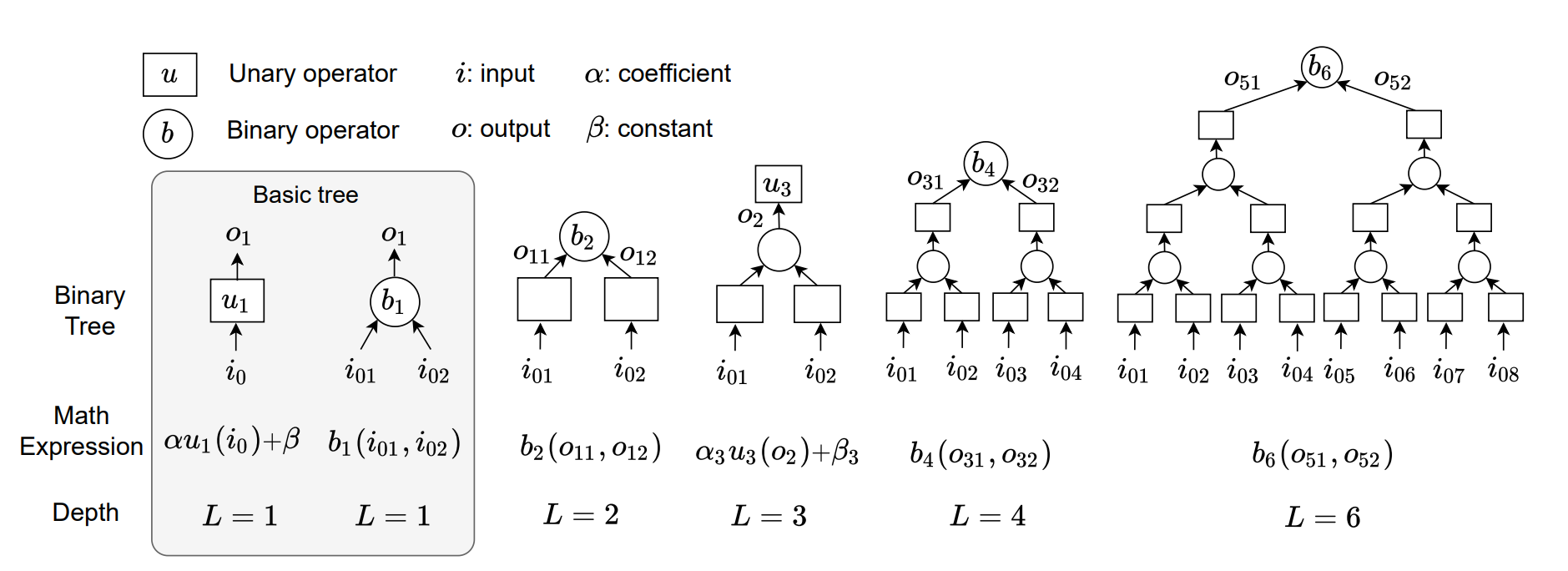}
    \caption{Binary-tree representation in FEX. Each node in $\mathcal{T}$ corresponds to a unary or binary operator. A depth-1 tree ( \(L = 1\)), consists of a single operator, while deeper trees ( \(L > 1\)) compute outputs recursively through hierarchical composition of operators.}
    \label{fig:FEX_tree}
\end{figure}

For each state component \(i = 1,\cdots, d\), we collect
the operators associated with all nodes of a fixed tree $\mathcal{T}$ into an operator sequence $\mathbf{e}_i$, together with their corresponding trainable parameters ${\omega}_{i}$. This yields a representation of a symbolic expression \(C_{i}(\cdot; \mathcal{T}, \mathbf{e}_{i}, {\omega}_{i}):\mathbb{R}^d \to \mathbb{R}\), which we abbreviate as \(C_i(\cdot)\) when the parameters are clear from context. 
For convenience, we denote the pair of discrete operators and continuous parameters by
\(
\Theta_i = (\mathbf{e}_i,\omega_i).
\)
The optimal symbolic expression for the \(i\)-th component is discovered by solving the   mixed optimization problem:
\begin{equation}\label{CO}
  \Theta_i^{*} = 
  \arg\min_{\Theta_i} \mathcal{L}(\Theta_i).
\end{equation}
Here \(\Theta_i^{*} = (\mathbf{e}_i^{*},\omega_i^{*})\) denotes the optimal operator sequence and its corresponding parameters. 
The empirical loss is defined as
\[
\mathcal{L}(\Theta_i)
=
\frac{1}{2MN}
\sum_{m=1}^{M}\sum_{n=0}^{N-1}
\left|
\Delta \mathbf{u}_i^{(m)}(t_n)
-
C_i(\mathbf{u}^{(m)}(t_n))
\right|^2.
\]
with the empirical time increment
\[\Delta \mathbf{u}_i^{(m)}(t_n) = \frac{{u}_{i}^{(m)}(t_{n+1})-{u}^{(m)}_{i}(t_n)}{h}.\]

After solving (2.1) for each state component $i = 1, \ldots, d$, the resulting vector-valued symbolic model is assembled as 
\(
C^*(\cdot) = \big[ C_i(\cdot; \mathcal{T}, e_i^*, \omega_i^*) \big]_{i=1}^d.
\)

\begin{remark}
The component-wise formulation reduces the complexity of symbolic discovery by decomposing a \(d\)-dimensional system into 
\(d\) lower-dimensional subproblems. This effectively leads to a reduced-order learning procedure at the level of each component, improving scalability and allowing dominant interactions to be identified more accurately before assembling the full system. A detailed component-wise strategy that further exploits the structure of the interactions in high dimensional phase space will be explored in future work.
\end{remark}

\subsection{Generative Models for Stochastic Error Correction}

Since FEX models only the deterministic components of the dynamics \eqref{GeneralTurbul}, two sources of error remain: the stochastic increments left unmodeled and the model error arising from the estimated coefficients in FEX. To address these deficiencies, generative models serve as a remedy by learning and correcting these residuals following the FEX approximation.
Previous FEX studies~\cite{liang5327407identifying,liu2025training} have demonstrated that TFDM is a promising candidate for capturing long-term statistical behaviors; here, we further examine its ability to recover high-order statistics as described below. In addition, SRAN and VAE are also introduced as standard baselines for comparison. For convenience, we define the residual at each time step $t_n$, $0 \le n \le N$, as given by:
\begin{equation}
\label{residual}
 \mathbf{R}_n = \mathbf{R}(\mathbf{u}(t_n)) = \mathbf{u}(t_{n+1}) - \mathbf{u}(t_n) - h\, C^*(\mathbf{u}(t_n)),   
\end{equation}
and denote the full residual
\(
\mathbf{R} = [\mathbf{R}_0, \ldots, \mathbf{R}_{N-1}]^T.
\)
\subsubsection{Training Free Diffusion Model} 
\label{GM}
To model the distribution of \(\mathbf{R}\), TFDM 
employs a forward--reverse SDE framework on the interval $[0,1]$. In the forward process, the residual is gradually transformed into a standard Gaussian distribution, while the reverse process reconstructs samples from this Gaussian noise.
Specifically, the forward dynamics is defined as
\begin{equation}\label{forwardSDE}
dZ_{\tau} =
b(\tau)\, Z_{\tau}\, d\tau
+
q(\tau)\, dW_{\tau},
\end{equation}
where \(W_{\tau}\) denotes a standard Brownian motion. The corresponding reverse-time dynamics is
\begin{equation}\label{reverseSDE}
dZ_{\tau} =
\left[
b(\tau)\, Z_{\tau}
-
q^2(\tau)\, V(Z_{\tau},\tau)
\right] d\tau
+
q(\tau)\, dB_{\tau},
\end{equation}
where \(B_{\tau}\) represents the reverse-time Brownian motion. The boundary conditions are given by 
$Z_0 \sim p_{{R}}$ and $Z_1 \sim \mathcal N(0,\mathbf{I}_d)$,
where $p_{R}$ denotes the empirical distribution of the residual samples $\mathbf{R}$, and \(Z_0, Z_1 \in \mathbb{R}^{d}\) have same dimension. The scalar drift and diffusion coefficients \(b(\tau)\) and \(q(\tau)\) are specified by
\[
b(\tau) = \frac{d}{d\tau}\log \alpha_{\tau},
\qquad
q^2(\tau)
=
\frac{d}{d\tau}\beta_{\tau}^2
-
2\,\beta_{\tau}^2\,\frac{d}{d\tau}\log \alpha_{\tau},
\]
where 
$\alpha_{\tau} = 1-\tau,
\beta_{\tau}^2 = \tau.$
  Denote  the
score of the probability density \(p_{\tau}(z)\) of \(Z_{\tau}\) by $V(Z_{\tau},\tau)=\nabla_{z}\log p_{\tau}(z)\big|_{z=Z_{\tau}}, $ where \(V(Z_{\tau},\tau)\) is evaluated numerically using the Monte Carlo estimator described in Sections~2 and~3 of~\cite{liu2025training}. \eqref{reverseSDE} leads to the following reverse-time
Fokker–Planck equation for the density \(p_{\tau}(z)\):
\begin{equation}\label{FP}
\frac{\partial p_{\tau}(z)}{\partial \tau}
=
\nabla_z \cdot
\left[
\left(b(\tau)z-q^2(\tau)V(z,\tau)\right)p_{\tau}(z)
+\frac{q^2(\tau)}{2}\nabla_z p_{\tau}(z)
\right],
\end{equation}
together with  
\[
\nabla_z p_{\tau}(z)=p_{\tau}(z)\nabla_z \log p_{\tau}(z)
= p_{\tau}(z)V(z,\tau).
\]
\eqref{FP}  can be rewritten as a Liouville equation describing a deterministic
probability flow, which admits the following
equivalent probability flow as an ordinary differential equation (ODE):
\begin{equation}\label{ODE}
dZ_\tau =
\left[
b(\tau)Z_{\tau}-\frac{1}{2}q^2(\tau)V(Z_{\tau},\tau)
\right]d\tau.
\end{equation}

The objective for TFDM in Stage II is to learn the reverse diffusion process \eqref{reverseSDE} that maps
Gaussian noise to the residual distribution. Specifically, 
we draw samples $\boldsymbol{\xi}_1 \sim \mathcal{N}(0, \mathbf{I}_d)$ and obtain the associated targets $\mathbf{y}$ by numerically integrating \eqref{ODE} from $\tau = 1$ to $\tau = 0$. This produces training pairs $(\boldsymbol{\xi}_1, \mathbf{y})$ which are used to train a neural network $\mathbf{G}_{\theta_1}(\boldsymbol{\xi}_1)$ with trainable parameters $\theta_1$
by minimizing the expected mean-squared error over the dataset:
\begin{equation}\label{optim_N2}
\theta_1^* =
\arg\min_{\theta_1}
\mathbb{E}_{\boldsymbol{\xi}_1}
\left\|
\mathbf{G}_{\theta_1}(\boldsymbol{\xi}_1) - \mathbf{y}
\right\|_F^2,
\end{equation}
where $\|\cdot\|_F$ denotes the Frobenius norm.
The resulting network $\mathbf{G}_{{\theta_1}^*}$ serves as the TFDM estimator.

\subsubsection{Stochastic Residual Activation Network}
We next introduce a second network \(\mathbf{G}_{\theta_2}\) with trainable parameters \(\theta_2\) by inputting a Gaussian random variable \(\boldsymbol{\xi}_2 \sim \mathcal{N}(0,\mathbf{I}_d)\) and matching the distribution of  \(\mathbf{R}\)  through the following to obtain the optimal parameters \(\theta_2^*\):
\begin{equation}\label{optim2}
\theta_2^* =
\arg\min_{\theta_2}
\;
\mathbb{E}_{\boldsymbol{\xi}_2}
\big\|
\mathbf{G}_{\theta_2}(\boldsymbol{\xi}_2)-\mathbf{R}
\big\|_F^2
+
 {\lambda}
\left\|\mathbb{E}_{\boldsymbol{\xi}_2}\big[\mathbf{G}_{\theta_2}(\boldsymbol{\xi}_2)
\mathbf{G}_{\theta_2}(\boldsymbol{\xi}_2)^\top\big]
-
h\mathbf{I}_d
\right\|_F^2,
\end{equation}
where $\lambda > 0$ is a regularization parameter. The first term is a standard least-squares objective that enforces sample-wise matching but only guaranties first-order statistical consistency. Since accurate recovery of the distribution of \(\mathbf{R}\) requires matching second-order statistics, the second term explicitly regularizes the covariance of the generated samples to match the expected scaling \(h\mathbf{I}_d\), ensuring that the generated samples capture the correct density behavior. The resulting network $\mathbf{G}_{\theta_2^*}$ serves as the SRAN estimator.

\subsubsection{Variational Autoencoder}
The VAE consists of an encoder network \(\mathbf{G}_{\phi_3}\) with trainable parameters \(\phi_3\) and a decoder network \(\mathbf{G}_{\theta_3}\) with trainable parameters \(\theta_3\). 
The encoder maps \(\mathbf{R}\) to the parameters \(
(\boldsymbol{\mu},\, \boldsymbol{\Sigma})
\) of a latent Gaussian distribution \(
q_{\phi_3}(\boldsymbol{\xi}_3 \mid \mathbf R)
=
\mathcal{N}\!\bigl(\boldsymbol{\mu},\, \boldsymbol{\Sigma}\bigr)
\), which represents the approximate posterior over the latent variable \(\boldsymbol{\xi}_3\) given the residual \(\mathbf{R}\), from which latent samples  are drawn.
The decoder~\(\mathbf{G}_{\theta_3}(\boldsymbol{\xi}_3)\) then reconstructs \(\mathbf{R}\) from these latent samples. The model is trained using the following weighted objective:
\begin{equation}
\begin{split}
(\theta_3^*, \phi_3^*) = \arg\min_{\theta_3, \phi_3} 
\Bigl[ \, &\alpha_{\text{mean}} \mathbb{E}_{\boldsymbol{\xi}_3} 
\left\| \mathbf{G}_{\theta_3}(\boldsymbol{\xi}_3) - \mathbf{R} \right\|_F^2 
+ \alpha_{\text{var}} \left\| \mathbb{E}_{\boldsymbol{\xi}_3}\left[ \mathbf{G}_{\theta_3}(\boldsymbol{\xi}_3) \mathbf{G}_{\theta_3}(\boldsymbol{\xi}_3)^\top\right] - h\mathbf{I}_d \right\|_F^2 \\
&+ \beta_{\text{KL}} \mathcal{L}_{\text{KL}} \,\Bigr],
\end{split}
\end{equation}
where \(\alpha_{\text{mean}}, \alpha_{\text{var}}, \beta_{\text{KL}} > 0\) are regularization parameters. The first term enforces sample-wise reconstruction of \(\mathbf{R}\); the second regularizes the covariance of the generated samples, and the third is the Kullback–Leibler divergence between the encoder distribution \(q_{\phi_3}(\boldsymbol{\xi}_3 \mid \mathbf{R})\) and the standard Gaussian prior \(p(\boldsymbol{\xi}_3) = \mathcal{N}(0, \mathbf{I}_d)\):
\[
\mathcal{L}_{\text{KL}} = D_{\text{KL}}\left(q_{\phi_3}(\boldsymbol{\xi}_3 \mid \mathbf{R}) \,\|\, p(\boldsymbol{\xi}_3)\right) = \frac{1}{2}\left(\boldsymbol{\mu}^\top \boldsymbol{\mu} + \mathrm{tr}(\boldsymbol{\Sigma}) - \log\det\boldsymbol{\Sigma} - d\right),
\]
which regularizes the latent distribution toward the prior. The resulting trained decoder \(\mathbf{G}_{\theta_3^*}(\cdot)\) serves as the VAE estimator.

\section{Numerical Algorithms for the Two-Stage Training and Inference}\label{ASD_Implementation} 

We next describe the implementation details of coupling FEX and the generative models in this two-stage framework. In Stage II, we use TFDM as an illustrative example. For SRAN and VAE, we also provide the pseudo-algorithms in Algorithms~SM2.1 and~SM2.2 in the Supplementary Material.

\subsection{Implementation of FEX}\label{FEX_search}  

In the discovery setting, we do not assume any prior knowledge of the functional forms or coupling structure of the dynamics. Instead, guided by the general mathematical understanding that dynamical systems can be characterized by linear, nonlinear, and time-dependent components, a unified binary tree structure is adapted in Figure~\ref{fig:tree-structure} to decompose turbulent dynamics~\eqref{GeneralTurbul} accordingly. Since the binary trees in FEX can represent a broad class of symbolic expressions, each component is learned directly from data: linear terms are identified automatically during the search process, nonlinear interactions are captured through multiple one-dimensional tree blocks extended across dimensions, and time-dependent forcing is modeled through a separate one-dimensional tree block with the temporal variable, without further structural assumptions. This design yields a low-complexity symbolic representation while maintaining sufficient expressivity for the target system.

\begin{figure}[!htbp]
\begin{center}
    \includegraphics[width=1\linewidth]{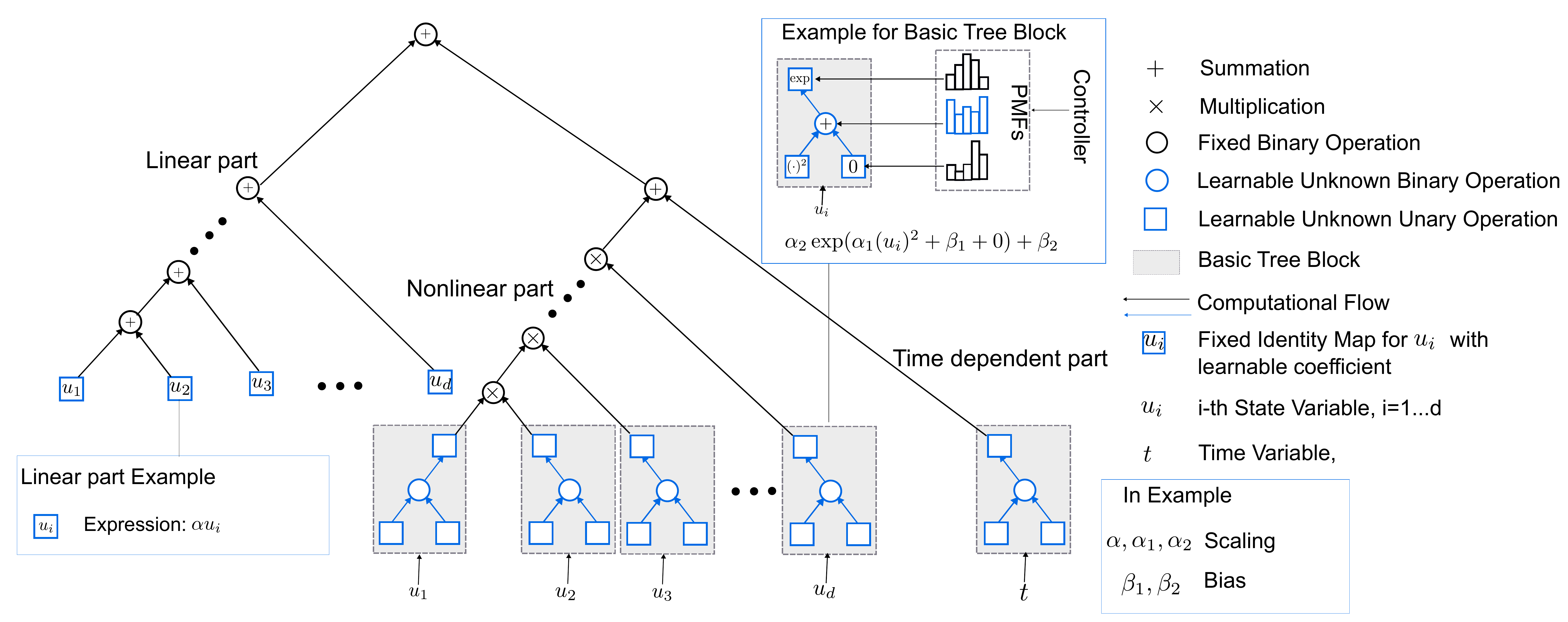}
    \caption{The structure of FEX binary tree which is formed by linear, nonlinear and time dependent part.}
\label{fig:tree-structure}
\end{center}
\end{figure}

In Figure~\ref{fig:tree-structure}, the binary operator can be chosen as 
\( \mathbb{B} = \{+, -, \times\} \), and the unary operator candidate comes from 
\(\mathbb{U} = \{\sin, \cos, \exp, 0, \mathrm{Id}, (\cdot)^2, (\cdot)^3, (\cdot)^4 \}\).
Following this configuration, for each output dimension \(i=1,\dots,d\) we recognize an aligned operator sequence 
\(\mathbf{e}_i = (e_{ij})_{j=1}^{d}\).
Each \(e_{ij}\) represents a tree block involving the \(j\)-th input dimension and contains four operator positions,
denoted by \(e_{ij} = (e_{ij1}, e_{ij2}, e_{ij3}, e_{ij4})\).
To find the optimal aligned operator sequence \(\mathbf{e}_i^{*}\) in \eqref{CO}, the implementation proceeds through four key components, which are described below.

 \paragraph{Operator Sequence Generation}
At each optimization iteration, a neural network, referred to as the controller $\chi_{\Phi}$ and parameterized by $\Phi$, generates candidate aligned operator sequences $\mathbf{e}_i$ to explore the search space. Specifically, for each position $(i,j,k)$, the controller outputs a probability mass function (PMF), i.e., a categorical distribution over the candidate operator set, from which the operator $e_{ijk}$ is sampled. To encourage exploration and avoid premature convergence, we adopt an $\epsilon$-greedy strategy: with probability $\epsilon$, $e_{ijk}$ is sampled uniformly at random, and with probability $1-\epsilon$, it is drawn according to the learned PMF.

\paragraph{Reward Computation}\label{3.1.2}

Given an aligned operator sequence $\mathbf{e}_i$, we evaluate its quality via the reward function
\[
\mathrm{Reward}(\mathbf{e}_i) = \left(1+\sqrt{L(\mathbf{e}_i)}\right)^{-1},
\]
where
\(
L(\mathbf{e}_i) = \min_{w_i} \left\{ \mathcal{L}\big(C(\cdot,\mathcal{T},\mathbf{e}_i,w_i)\big) \right\}
\)
denotes the minimum loss over $w_i$ for a fixed operator sequence $\mathbf{e}_i$, as defined in \eqref{CO}.
A higher reward indicates that the symbolic expression defined by $\mathbf{e}_i$ provides a better approximation of the target dynamics, corresponding to smaller values of $L(\mathbf{e}_i)$.
However, computing the exact minimum of $L(\mathbf{e}_i)$ is computationally expensive. Since the search stage only requires comparing candidate operator sequences, we typically trade off accuracy for efficiency by approximating $L(\mathbf{e}_i)$.

\paragraph{Controller Update}

To improve the probability of sampling operator sequences that better approximate the target dynamics, the controller distribution is updated based on the scores of the sampled candidates in Section \ref{3.1.2}. FEX adopts a policy gradient approach to maximize the following objective function~\cite{petersen2019deep}:
\begin{equation}
    \mathcal{J}(\Phi)
    =
    \mathbb{E}_{\mathbf{e}\sim \chi_{\Phi}}
    \left\{
    \mathrm{Reward}(\mathbf{e})
    \mid
    \mathrm{Reward}(\mathbf{e}) \ge \mathrm{Reward}_{v,\Phi}
    \right\}.
\end{equation}
Here $\mathrm{Reward}_{v,\Phi}$ denotes the $(1-v)\times100\%$ quantile of the score distribution under $\chi_{\Phi}$. Only sequences within the top \(v\)-fraction are used to estimate the policy gradient, which is used to update the parameter \(\Phi\) of the controller.

\paragraph{Candidate Selection}
To retain promising operator sequences discovered during the search, FEX maintains a candidate pool $\mathcal{C}$, where each sequence subsequently needs further refinement by fine-tuning the parameters using a first-order optimization method with a small learning rate. This refinement improves parameter accuracy and reduces the risk of convergence to poor local minima. Finally, FEX outputs a set of refined candidates from $\mathcal{C}$ for subsequent selection.

\subsection{Implementation of TFDM}\label{DesSZ}

An important aspect of TFDM is to efficiently recover the distribution of the residual $\mathbf{R}$. Using all samples directly is computationally costly, so we instead characterize the local distribution via a nearest-neighbor search. Specifically, 
a GPU-based implementation package, FAISS, is employed here to identify a set of neighboring samples for each 
training point, forming local clusters that represent similar residual behaviors. These 
neighborhood indices are used to construct reduced local residual datasets for TFDM 
training. 

Regarding computational cost, we note that TFDM requires training at each evaluation step to learn the residual distribution, meaning the total cost scales with the number of time steps and is therefore sensitive to the choice of step size. Furthermore, while accurate residual recovery can be achieved with relatively few samples, incorporating more data points into training improves the quality of the learned distribution at the expense of increased computational cost. These trade-offs between step size selection, data volume, and computational efficiency are important practical considerations, and further investigation into more scalable TFDM implementations is left for future work.

Having described each stage independently, we now clarify how they are coupled in practice. The two stages are linked through the residual \(\mathbf{R}\): once FEX identifies the symbolic representation \(C^*\) in Stage I, the residual is computed and passed to the generative model in Stage II for training. At inference, the learned symbolic dynamics and the generative model are combined to produce the full estimated trajectory. The complete training and inference procedures are summarized in Algorithm~\ref{alg:training_ASD_FEX_TFDM_independent} and Algorithm~\ref{alg:inference_ASD_FEX_TFDM}, respectively.

\begin{algorithm}[!htbp]
\caption{Training stage for FEX in Stage I with TFDM in Stage II}
\label{alg:training_ASD_FEX_TFDM_independent}

\textbf{Input:} Observed state trajectories 
$\{\mathbf{u}^{(m)}(t_n)\}_{m,n=1}^{M,N}$,
where $\mathbf{u}^{(m)}(t_n)\in\mathbb{R}^d$; index $K$.\\
\textbf{Output:} Estimated symbolic model $C^{*}$; estimated network model \(\mathcal{N}_{\theta_1}(\cdot)\) with the optimal parameters \(\theta_1^{*}\).

\begin{algorithmic}[1]

\For{$i = 1, \dots, d$}
    \State Minimize $\mathcal{L}(\Theta_i)$ via~\eqref{CO} following Section~\ref{FEX_search}.
\EndFor

\State Construct $\mathbf{C}^{*}$; then compute the residual $\mathbf{R}$ via~\eqref{residual}.

\State Apply FAISS-based nearest-neighbor selection to obtain a representative sample $Z_0$ from $\mathbf{R}$.

\State Sample $Z_1 \sim \mathcal{N}(0,\mathbf{I}_d)$ and set $Z_{\tau_K} = Z_1$.

\For{$k = K,\dots,1$}
    \State Estimate the score function $V(Z_{\tau_k},\tau_k)$ using the Monte Carlo method in~\cite{liu2025training}.
    \State Integrate the reverse dynamics in~\eqref{ODE} to obtain $Z_{\tau_{k-1}}$.
\EndFor

\State Set $\mathbf{y} = Z_{\tau_0}$ and construct the training pairs $(\boldsymbol{\xi}_1, \mathbf{y})$.

\State Train the neural network $\mathcal{N}_{\theta_1}(\boldsymbol{\xi}_1)$ using $(\boldsymbol{\xi}_1, \mathbf{y})$ by solving~\eqref{optim_N2}.

\end{algorithmic}
\end{algorithm}

\begin{algorithm}[!htbp]
\caption{Inference procedure for FEX in Stage I with TFDM in Stage II}
\label{alg:inference_ASD_FEX_TFDM}

\textbf{Input:} Learned symbolic model $C^{*}$; pretrained TFDM model $\mathcal{N}_{\theta_1^{*}}$; initial distribution $\rho_0$; horizon $T$; step size $h$; number of trajectories $M$.\\
\textbf{Output:} Predicted trajectories $\{\hat{\mathbf{u}}^{(m)}(t_n)\}_{m=1}^{M}$.

\begin{algorithmic}[1]
\State Set $N_1=\lfloor T/h \rfloor$.
\State Draw initial states $\hat{\mathbf{u}}^{(m)}(0) \sim \rho_0$, for $m=1,\ldots,M$.

\For{$n = 1,\ldots,N_1$}
    \For{$m = 1,\ldots,M$}
      \State Compute the deterministic increment $\Delta \hat{\mathbf{u}}_{\mathrm{det}} = h\, C^{*}(\hat{\mathbf{u}}^{(m)}, t_{n-1})$.

\State Sample a Gaussian input $\boldsymbol{\xi}^{(m)}_n \sim \mathcal{N}(\mathbf{0}, \mathbf{I}_d)$ to drive the stochastic component.

\State Predict the residual increment $\widehat{\mathbf{R}}^{(m)}_n = \mathcal{N}_{\theta_1^{*}}(\boldsymbol{\xi}^{(m)}_n, \hat{\mathbf{u}}^{(m)}, t_{n-1})$ using TFDM.

\State Update the state by combining deterministic and residual increments:
$\hat{\mathbf{u}}^{(m)} \leftarrow \hat{\mathbf{u}}^{(m)} + \Delta \hat{\mathbf{u}}_{\mathrm{det}} + \widehat{\mathbf{R}}^{(m)}_n$.
    \EndFor
\EndFor

\State \Return $\{\hat{\mathbf{u}}^{(m)}(t_n)\}_{m=1}^{M}$.
\end{algorithmic}
\end{algorithm}

\section{Theoretical  Error Estimates of the Two-Stage Approach}\label{Numerical_Analysis}
We demonstrate the theoretical properties of the two-stage framework, taking FEX and TFDM as representative components. The analysis proceeds in two stages. Given that the optimal operator sequence has already been found and is treated as fixed, in Stage I we establish that the FEX coefficient estimator is consistent and quantify its convergence rate. The key ingredient is an ergodicity assumption that ensures empirical averages along trajectories converge to expectations with respect to an invariant measure, turning the coefficient estimation into a well-posed regression problem. Consistency then follows from the stochastic increments averaging out and the approximation residual being made arbitrarily small, yielding a standard bias-variance tradeoff. In Stage II, we analyze the TFDM estimator by adopting a Gaussian mixture initialization, which provides a smoothed and well-behaved starting distribution for the reverse diffusion process. The total error decomposes into a smoothing error, a discretization error, and a network approximation error, with explicit dependence on the step size and dimension.

\begin{assumption}[Ergodicity \cite{gu2023stationary,wanner2024higher}]\label{Assumption1}
  The stochastic process $\{ \mathbf{u}(t) \}_{t \ge 0}$ admits an ergodic invariant measure $\rho$. In particular, for any integrable function $f$, and discrete time instants, $t_1<t_2<\dots<t_n$ with $t_n\rightarrow\infty$ as $n\rightarrow\infty$
\[
\lim_{T \to \infty} \frac{1}{T} \int_0^T f(\mathbf{u}(t)) \, dt 
= \lim_{N \to \infty} \frac{1}{N} \sum_{n=0}^{N-1} f(\mathbf{u}(t_n)) 
= \int_{\mathbb{R}^d} f(x)\, d\rho(x), \text{ a.s.}.
\]  
\end{assumption}

With this ergodic measure, we also naturally consider the Hilbert space $L^2(\rho)$. For any $f, g \in L^2(\rho)$, the inner product can be evaluated via time averages:
\[
\langle f, g \rangle 
= \int_{\mathbb{R}^d} f(x) g(x)\, d\rho(x) 
= \lim_{T \to \infty} \frac{1}{T} \int_0^T f(\mathbf{u}(t)) g(\mathbf{u}(t))\, dt 
= \lim_{N \to \infty} \frac{1}{N} \sum_{n=0}^{N-1} f(\mathbf{u}(t_n)) g(\mathbf{u}(t_n)).
\]

While ergodicity characterizes the global statistical behavior, the framework also requires a local description of the one-step evolution \eqref{FM1}. 
To accommodate more general stochastic dynamics beyond the additive noise case in~\eqref{FM1}, we model the system as an It\^o diffusion and assume that the stochastic increment admits a first-order It\^o–Taylor expansion.

\begin{assumption}[Strong It\^o--Taylor expansion \cite{wanner2024higher}]\label{Assumption4.3}
For the one-step decomposition in \eqref{FM1}, assume that the stochastic increment
\(
\mathbf S(t,z;h)=\bigl(S_1(t,z;h),\dots,S_d(t,z;h)\bigr)^\top
\)
admits the first-order strong It\^o--Taylor representation
\[
S_i(t,z;h)
=
\sum_{m=1}^d \sigma_{im}(\mathbf u(t),t)\,\sqrt{h}\,z_m
+
R_i(t,h),
\qquad i=1,\dots,d,
\]
where $z=(z_1,\dots,z_d)^\top\sim\mathcal N(0,\mathbf I_d)$, and the remainder \(R_i(t,h)\) satisfies
\[
\mathbb E\!\left[R_i(t,h)\mid \mathbf u(t)\right]=0,
\qquad
\mathbb E\!\left[|R_i(t,h)|^2\mid \mathbf u(t)\right]=O(h^2).
\]
\end{assumption}

\subsection{Stage I for FEX}
Although FEX represents the underlying dynamics via a compositional binary-tree structure in Figure~\ref{fig:tree-structure}, 
direct parameter-wise error analysis is intractable due to its hierarchical representation. 
So we assume that FEX successfully recovers the exact operator form of the nonlinear interaction term $\mathbf B(\mathbf u,\mathbf u)$, and focus on estimating the associated coefficients. 
This allows us to reformulate the learned model as a linear regression problem over a set of basis functions, thereby reducing the task to finite-dimensional coefficient estimation. 
To ensure the validity of this formulation, we also impose the following assumptions.

\begin{assumption}[Linear independence of basis functions]\label{Assumption4.4}
For each component \(i=1,\dots,d\), the following set of basis functions is assumed to be linearly independent in \(L^2(\rho)\):

(i) the state components $u_1,\dots,u_d$ from \(\mathbf{u}\), 

(ii) the $i$-th component of the nonlinear interaction term $\mathbf B(\mathbf u,\mathbf u)$, 

(iii) time-dependent polynomial functions $\{1,t,\dots,t^{n_1}\}$ for some $n_1>0$. 
\end{assumption}

Under  Assumption~\ref{Assumption4.4} (i)-(iii),  each component \( C_i(\mathbf u,t) \) of the deterministic part \( C(\mathbf u,t) \) in \eqref{GeneralTurbul}, for \( i=1,\dots,d \), admits a decomposition in terms of basis functions:
\begin{equation}\label{FEX_decomposition}
C_i(\mathbf u,t)
=
\sum_{r=1}^{K} \alpha_r^i \, \psi_r^i(\mathbf u,t) + r_i(t),
\end{equation}
where \( \alpha^i = (\alpha_1^i,\dots,\alpha_{K}^i)^\top \in \mathbb{R}^{K} \) denotes the coefficient vector, and \( \{ \psi_r^i \}_{r=1}^{K} \) denotes the set of functions.
Note that the coefficient vector \( \alpha^i \) is not identical to the original FEX parameters but is induced through the FEX tree. 
The term $r_i(t)$ represents the approximation residual associated with the time-dependent forcing. We assume that this residual can be made arbitrarily small, which is formalized as follows.
\begin{assumption}[Arbitrarily small residual]\label{Assumption4.5}
For each $i=1,\dots,d$ and every $\varepsilon_i>0$, there exists a choice of polynomial time-dependent basis terms $\{1,t,\dots,t^{n_1}\}$ with $n_1>0$ in \eqref{FEX_decomposition} such that the residual $r_i(t)$ satisfies
\[
\|r_i\|_{\infty,[0,T]} \le \varepsilon_i,
\qquad
\text{where }
\|r_i\|_{\infty,[0,T]}:=\sup_{t\in[0,T]}|r_i(t)|.
\]
\end{assumption}

Accordingly, the least-squares estimator in \eqref{CO} reduces to the following optimization problem:
\begin{equation}\label{easyCO}
\hat{\alpha}^i
=
\arg\min_{\alpha \in \mathbb{R}^{M_i}}
\frac{1}{2MN}
\sum_{m=1}^M \sum_{n=0}^{N-1}
\left|
\Delta u_i^{(m)}(t_n)
-
\sum_{r=1}^{K} \alpha_r \psi_r^i(\mathbf u^{(m)}(t_n),t_n)
\right|^2.
\end{equation}

We are now ready to state the following theorem, which establishes the asymptotic convergence of the estimator.

\begin{theorem}[Asymptotic convergence in probability]\label{Thm:alpha_convergence}
Under Assumptions~\ref{Assumption1}--
\ref{Assumption4.5},
let $\hat{\alpha}^i$ be the least-squares estimator defined in \eqref{easyCO}, 
constructed from $M$ independent trajectories of length $T$ with time step $h$. 
Under the limits $T \to \infty$ and $h \to 0$, we have 
\(
\hat{\alpha}^i \xrightarrow{P} \alpha^i
\), where $\alpha^i$ is the true coefficient in \eqref{FEX_decomposition}.
\end{theorem}

\begin{proof}
For each fixed $i=1,\dots,d$, define
\[
Y_{m,n}^i:=\frac{\Delta u_i^{(m)}(t_n)}{h}, \qquad
\Psi_{m,n}^i
:=
\bigl(\psi_1^i(\mathbf u^{(m)}(t_n),t_n),\dots,\psi_{K}^i(\mathbf u^{(m)}(t_n),t_n)\bigr)^\top .
\]
With these definitions, the least-squares estimator in \eqref{easyCO} satisfies the normal equation
\[
\hat\alpha^i
=
\Bigl(\frac{1}{MN}\sum_{m=1}^M\sum_{n=0}^{N-1}\Psi_{m,n}^i(\Psi_{m,n}^i)^\top\Bigr)^{-1}
\Bigl(\frac{1}{MN}\sum_{m=1}^M\sum_{n=0}^{N-1}\Psi_{m,n}^i\,Y_{m,n}^i\Bigr),
\]
provided the empirical Gram matrix is invertible.

By the linear independence of \(\{\psi_r^i\}_{r=1}^K\) in \(L^2(\rho)\)
from Assumption~\ref{Assumption4.4}, the Gram matrix \(G_i\) with entries \((G_i)_{rs} := \langle \psi_r^i, \psi_s^i\rangle\) is positive definite and hence invertible.
Moreover, by the convergence of empirical averages under Assumption~\ref{Assumption1},
\[
\frac{1}{MN}\sum_{m=1}^M\sum_{n=0}^{N-1}\Psi_{m,n}^i(\Psi_{m,n}^i)^\top
\xrightarrow{P} G_i .
\]

Next, by \eqref{FM1} and \eqref{FEX_decomposition},
\[
Y_{m,n}^i
=
(\alpha^i)^\top \Psi_{m,n}^i
+
{r_i(t_n)}
+
\frac{S_i(t_n,z;h)}{h}.
\]
Therefore,
\[
\frac{1}{MN}\sum_{m=1}^M\sum_{n=0}^{N-1}\Psi_{m,n}^i\,Y_{m,n}^i
=
\Bigl(\frac{1}{MN}\sum_{m=1}^M\sum_{n=0}^{N-1}\Psi_{m,n}^i(\Psi_{m,n}^i)^\top\Bigr)\alpha^i
+
\mathcal R_{M,N}^i
+
\mathcal S_{M,N}^i,
\]
where
\[
\mathcal R_{M,N}^i
:=
\frac{1}{MN}\sum_{m=1}^M\sum_{n=0}^{N-1}\Psi_{m,n}^i\,{r_i(t_n)},
\qquad
\mathcal S_{M,N}^i
:=
\frac{1}{MN}\sum_{m=1}^M\sum_{n=0}^{N-1}\Psi_{m,n}^i\,\frac{S_i(t_n,z;h)}{h}.
\]

By the arbitrarily small bound of \(r_i\)
in Assumption~\ref{Assumption4.5}, 
$\mathcal R_{M,N}^i$ is negligible in the limit. From the first condition of Assumption~\ref{Assumption4.3}, the term \(\Psi_{m,n}^i \cdot S_i/h\) have zero conditional mean, and the variance of the average \(\mathcal{S}_{M,N}^i\)
is of order \(O(1/MNh) = O(1/(MT))\), which vanishes as \(T\to\infty\). Hence, by Chebyshev's inequality, \(\mathcal S_{M,N}^i \xrightarrow{P} 0\).

It follows that
\[
\frac{1}{MN}\sum_{m=1}^M\sum_{n=0}^{N-1}\Psi_{m,n}^i\,Y_{m,n}^i
\xrightarrow{P}
G_i\alpha^i .
\]
Combining this with the convergence of the empirical Gram matrix and using continuity of matrix inversion on nonsingular matrices, we obtain
\[
\hat\alpha^i
\xrightarrow{P}
G_i^{-1}(G_i\alpha^i)
=
\alpha^i .
\]
This proves the theorem.
\end{proof}

To further quantify the estimation error, define the normalized one-step regression error
\begin{equation}\label{regression_error}
 \xi_i^{(m)}(t_n)
:=
\frac{\Delta u_i^{(m)}(t_n)}{h}
-
C_i(\mathbf u^{(m)}(t_n),t_n),\qquad
m=1,\dots,M,\quad n=0,\dots,N-1.   
\end{equation}

Let \(\Gamma^i\in\mathbb R^{MN\times K}\) denote the design matrix whose \((m,n)\)-th row is \((\Psi_{m,n}^i)^\top\), and let \(\xi^i\in\mathbb R^{MN}\) be the stacked vector with entries \(\xi^i_{(m,n)}=\xi_i^{(m)}(t_n)\). In this formulation, the coefficient error can be written in closed form, yielding the following estimate.

\begin{theorem}[Coefficient error bound for FEX]\label{thm:coef_error}
Under Assumptions~\ref{Assumption1}--\ref{Assumption4.5}, for each component \(i=1,\dots,d\), the least-squares estimator satisfies
\[
\hat\alpha^i-\alpha^i
=
\left(G_i^{-1}+o(1)\right)
\left(
\frac{1}{MN}(\Gamma^i)^\top (r^i+\xi^i)
\right),
\]
where \(G_i\) is the population Gram matrix with entries \((G_i)_{rs} := \langle \psi_r^i, \psi_s^i\rangle\), and $o(1)$ denotes a term that converges to zero in probability as the sample size $MN \to \infty$. In particular, if the approximation residual is chosen so that its contribution is negligible and the It\^o--Taylor remainder is of order \(h\), 
\[
\bigl\|\mathbb E[\hat\alpha^i-\alpha^i]\bigr\|
\le
C_i\bigl(\varepsilon_i+h+o(1)\bigr).
\]
while its variance satisfies \[ \mathrm{Var}(\hat\alpha^i) \le \frac{\widetilde C_i}{T}\bigl(1+o(1)\bigr), \qquad T=Nh, \] for some constants \(C_i,\widetilde C_i>0\) depending only on \(G_i\) and the chosen basis.
\end{theorem}

\begin{proof}
See \cite{frishman2020learning,gu2023stationary,wanner2024higher} for related proofs. 
Full details are provided in Supplementary Material SM1.1.1.
\end{proof}

\subsection{Stage II for TFDM}
To model the stochastic component, we need to characterize the distribution of the residual variable. But the empirical residual dataset is discrete and typically suffers from limited sample size, which may lead to variance underestimation and poor generalization. To address this issue, we adopt a Gaussian mixture initialization, which can be viewed as an inflation strategy \cite{wang2026error} widely used in uncertainty quantification. This approach introduces controlled smoothing to the empirical residual distribution, preventing variance collapse while ensuring a well-behaved score function throughout the reverse diffusion process.
 
\begin{assumption}[Gaussian mixture initialization]\label{Assumption4.6}
Given the residual $\mathbf{R} = [\mathbf{R}_0, \ldots, \mathbf{R}_{N-1}]^T $, 
we assume that the probability density of the residual variable $Z_0$ in ODE flow \eqref{ODE} at $t=0$ is given by
\begin{equation}\label{GMM_residual}
q_{Z_0}(z_0)
=
\sum_{n=0}^{N-1} P_\xi(n)\, p_{Z_0\mid \xi}(z_0\mid n)
=
\frac{1}{N}\sum_{n=0}^{N-1}\phi(z_0;\mathbf{R}_n,\boldsymbol{\Sigma}),
\end{equation}
where $\xi$ is a discrete random variable taking values in $\{0,\dots,N-1\}$ with 
$P_\xi(n)=1/N$, and
\(
p_{Z_0\mid \xi}(z_0\mid n)=\phi(z_0;\mathbf{R}_n,\boldsymbol{\Sigma}).
\)
Here, $\boldsymbol{\Sigma}$ is a positive definite covariance matrix that provides a Gaussian smoothing of the empirical residual distribution.
\end{assumption}

As a consequence, the residual samples admit a uniform bound, as stated in the following lemma.

\begin{lemma}[Bounded training residual set]\label{lemma:bounded_training_residual}
Under Assumptions~\ref{Assumption1}--\ref{Assumption4.5}, there exists a constant $M_R<\infty$ such that
\[
\|\mathbf{R}_n\|_2 \le M_R,\qquad n=0,\dots,N-1.
\]
Consequently, the residual training dataset $\mathbf R=[\mathbf{R}_0,\dots,\mathbf{R}_{N-1}]^\top$ is contained in an $\ell_2$-ball of radius $M_R$.
\end{lemma}

\begin{proof}
 Full details are provided in the Supplementary Material SM1.2.1.
\end{proof}

From this, we obtain the following error bound.

\begin{theorem}
Let $\mathbf y^{*}(\boldsymbol{\xi}_1)$ denote the continuous exact solution of \eqref{ODE} associated with input $\boldsymbol{\xi}_1 \sim \mathcal{N}(0,\mathbf I_d)$ on the closed interval \([0,1]\). 
Assume that the initialization corresponds to a Gaussian smoothing with covariance matrix $\Sigma = C_0^2 \mathbf I_d$ for some constant $C_0>0$. 
Let $n_2$ denote the number of training pairs $(\boldsymbol{\xi}_1,\mathbf y)$, and let $\mathcal{N}_{\theta_1}(\boldsymbol{\xi}_1)$ be the neural network in TFDM with trainable parameters $\theta_1$.
Assume further that the network results is bounded (i.e. 
\(
\mathbb{E}_{\boldsymbol{\xi}_1}\!\left[\left\|\mathcal{N}_{\theta_1}(\boldsymbol{\xi}_1)\right\|_2\right] < C_2.
\) for some \(C_2>0\).)
The total error satisfies
\[
\mathbb E_{\boldsymbol{\xi}_1}\!\left[
\|\mathbf y^{*}(\boldsymbol{\xi}_1)-\mathcal N_{\theta_1}(\boldsymbol{\xi}_1)\|_2
\right]
\le
C_0\sqrt d + C_1 d h + C_3,
\]
where $C_0, C_1, C_3 > 0$ are constants independent of step size $h$ and dimension $d$.
\end{theorem}

\begin{proof}
 Full details are provided in the Supplementary Material SM1.2.2.
\end{proof}

\section{Numerical Experiments}\label{Numerical_Experiment} 
For experimental evaluations, we adopt the stochastic triad model with state $\mathbf{u}=(u_1,u_2,u_3)$, a three-dimensional dynamical system obtained as a Galerkin truncation of a broader class of nonlinear systems. More specifically, it serves as a structured three-dimensional realization of \eqref{GeneralTurbul}, where the linear operators and quadratic interaction terms are explicitly parameterized. Despite its low dimensionality, this triad model retains the essential mechanism of nonlinear energy exchange among interacting modes and therefore provides a canonical testbed for investigating turbulent energy transfer and statistical behavior~\cite{gluhovsky1997interpretation,majda2002priori,majda2016introduction,randall2007climate}. The system is given by
\begin{equation}\label{Triad}
\left\{\begin{aligned}
\frac{du_1}{dt} &= L_2 u_3 - L_3 u_2 - d_1 u_1 + B_1 u_2 u_3 + F_1(t) + \sigma_1 \dot{W}_1, \\
\frac{du_2}{dt} &= L_3 u_1 - L_1 u_3 - d_2 u_2 + B_2 u_3 u_1 + F_2(t) + \sigma_2 \dot{W}_2, \\
\frac{du_3}{dt} &= L_1 u_2 - L_2 u_1 - d_3 u_3 + B_3 u_1 u_2 + F_3(t) + \sigma_3 \dot{W}_3.
\end{aligned}\right.
\end{equation}
The linear terms in \eqref{Triad} correspond to the operators in \eqref{GeneralTurbul}, with matrix representations
\[
\mathbf{L} =
\begin{bmatrix}
0 & -L_3 & L_2\\
L_3 & 0 & -L_1\\
-L_2 & L_1 & 0
\end{bmatrix},
\qquad
\mathbf{D} =
\begin{bmatrix}
-d_1 & 0 & 0\\
0 & -d_2 & 0\\
0 & 0 & -d_3
\end{bmatrix}.
\]
The quadratic nonlinear interaction term is
\(
\mathbf{B}(\mathbf{u},\mathbf{u})
=
(B_1u_2u_3,\; B_2u_3u_1,\; B_3u_1u_2)^\top,
\)
where the constraint $B_1+B_2+B_3=0$ guaranties nonlinear energy conservation in the absence of damping and forcing. The remaining terms represent external forcing \(F_i(t)\),  together with additive white-noise forcing \(\sigma_i  \dot{W}_i\), for \(i=1,2,3\). Different choices of the parameters \(d_i\), \(F_i(t)\), and \(\sigma_i  \dot{W}_i\) lead to qualitatively distinct stationary statistics, ranging from approximately Gaussian equilibria to strongly non-Gaussian regimes with pronounced skewness. 

Importantly, the triad system is adopted in numerical experiments since it captures the essential linear coupling and quadratic energy-conserving interactions of turbulence in a minimal setting. Relative to high-dimensional turbulent systems, this design reduces computational costs, alleviates attractor complexity, and enables efficient statistical sampling. The resulting low-dimensional structure improves identifiability and makes symbolic discovery substantially more tractable without sacrificing the fundamental energy-transfer mechanism. In this section, we focus on the following representative regimes demonstrating distinctive statistical features:

\textbf{Regime I: Equipartition of energy.} It is designed so that the system admits a unique stable equilibrium with approximate energy equipartition among the three modes, leading to near-Gaussian statistics around this steady state. The parameters are chosen as \(d_1 = 0.2, d_2 = 0.1, d_3 = 0.1,\) and \(B_1 = 1, B_2 = -0.6, B_3 = -0.4,\) with skew-symmetric linear coefficients \(L_1 = 3, L_2 = 2, L_3 = -1\), and no deterministic forcing is applied.

\textbf{Regime II: Forward energy cascade.} Here, the parameter hierarchy concentrates stochastic energy injection in the first mode \(u_1\) while enhancing dissipation in the remaining two modes. Specifically, we choose \(d_1 = 1, d_2 =2, d_3 =2\), and \(\sigma_1^2 = 10, \sigma_2^2 = 0.01, \sigma_3^2 = 0.01.\) Nonlinear coupling is taken as \(B_1 = 2, B_2 = B_3 =-1.\) Skew-symmetric linear interactions are removed by setting \(L_1 = L_2 = L_3= 0,\), and no deterministic forcing is applied.

\textbf{Regime III: Dual energy cascade.} The stochastic injection hierarchy is reversed by setting
\(
\sigma_1^2 = 0.01, \sigma_2^2 = \sigma_3^2 = 0.1,
\)
while keeping the same damping coefficients
and the nonlinear coupling coefficients as in the forward cascade case. Skew-symmetric linear interactions are introduced,
\(
L_1 = 0.09, ~L_2 = 0.06, ~L_3 = -0.03,
\)
to enhance inter-mode energy exchange. And constant deterministic forcing \(
F_1(t) = 0, F_2(t) = -1, F_3(t) = 1,\)
is applied here. For this configuration, quadratic and skew-symmetric interactions redistribute energy across all three modes, producing bidirectional transfer between 
\(u_1\) ,\(u_2\) and \(u_3\), and resulting in strongly skewed non-Gaussian stationary statistics..

\textbf{Regime IV: Energy cascade with periodic forcing.}
The same damping, noise, and nonlinear coupling parameters as in the forward energy-cascade regime are adopted, with vanishing skew-symmetric interactions.
The distinction in this regime is the introduction of coherent periodic forcing applied uniformly to all three modes \(F_1(t) = F_2(t)= F_3(t) = \sin(\frac{\pi}{4} t).\)
This periodic forcing generates time-dependent energy injection, which is redistributed through quadratic interactions, leading to a periodically modulated cascade and enabling the study of non-stationary energy transfer and phase-dependent flux dynamics.

\textbf{Regime V: Energy cascade with random forcing.} The same damping, nonlinear coupling, and skew-symmetric coefficients remain as in the periodic cascade regime, but replace the deterministic periodic forcing with a stochastic process applied uniformly to all three modes:\(
F_1(t) = F_2(t) = F_3(t) = m(t),
\)
where \(m(t)\) is an Ornstein–Uhlenbeck process satisfying
\(
dm = -5 m \, dt + 0.2 \, dW(t),
\) and injects energy into all modes in a time-varying manner. Unlike the periodic regime, both the direction and intensity of energy transfer fluctuate irregularly in time, capturing cascade dynamics under unsteady stochastic driving.

For brevity in the following discussions, Regimes I–V are referred to as Equipartition, Forward Cascade, Dual Cascade, Periodic Cascade, and Random Cascade, respectively.

\subsection{Implementation and Hyperparameter Settings}
Below, we summarize the implementation details of the model setup and the experimental hyperparameters used in all regimes.

\vspace{2mm}

\textbf{Sample Generation.} 
For each regime, trajectories of the stochastic triad model \eqref{Triad} are generated using the Euler–Maruyama discretization scheme. The initial condition is sampled from a multivariate normal distribution with mean 
\((-1,\; 0.5,\; -0.5)\) and variances \((0.5^2,\; 0.2^2,\; 0.1^2)\). A total of $M = 1{,}000$ training trajectories are simulated. Each trajectory consists of $N = 1{,}001$ time points obtained via uniform time discretization with step size $h = 0.01$ over a time horizon $T = 10.0$. Following Section~2, the trajectory data are reorganized into $M\times N$ scalar input–output pairs for each mode,
\(
\big(u_i^{(m)}(t_n),\, u_i^{(m)}(t_{n+1})\big) \in \mathbb{R}^2,
\quad i=1,2,3,
\)
where $m=1,\dots,M$ and $n=0,\dots,N-1$. These samples are used to train the framework. For evaluation, we generate an additional \(500{,}000\) test trajectories from the same regime and initial distribution, using a step size \(h = 0.01\). While this framework actually allows for arbitrary time horizons (such as \(T = 20, 50, 100\) or beyond), we fix \(T_{\text{test}} = 20.0\) in this experiment to evaluate the long-term statistical behavior.

\vspace{2mm}

\textbf{Setting of FEX.}
For all regimes, the binary tree architecture and the associated unary and binary operator sets follow the configuration described in Section~\ref{FEX_search}. We therefore only summarize the additional implementation details and hyperparameter settings of the four main components of FEX: \textcircled{\small{1}} \textbf{Score Computation.} The score function is first optimized using Adam with a learning rate of $8.0\times10^{-3}$ for 20 iterations, followed by 10 iterations of L-BFGS to solve  \eqref{CO}. \textcircled{\small{2}} \textbf{Controller Architecture.} A one-layer NN with ReLU is used here to represent the controller. Its output size is $|\mathbb{B}|+3|\mathbb{U}|$, where $|\cdot|$ denotes the cardinality of a set. \textcircled{\small{3}}  \textbf{Controller Update.} The controller parameters are updated using the policy gradient method with a batch size of 200, optimized via Adam with a learning rate of $2.0\times10^{-3}$. The number of training iterations varies across regimes. For cascade-type regimes, which exhibit stronger nonlinear interactions and more complex statistical behavior, up to 100 iterations are used. For simpler regimes, fewer iterations are sufficient. \textcircled{\small{4}} \textbf{Candidate Optimization.} To prevent incorrect expressions from replacing correct candidates within the training domain, the candidate pool size is set to $N=100$. After symbolic expressions are learned independently for each mode, they are concatenated and jointly fine-tuned. The parameters are optimized using Adam with an initial learning rate of $8.0\times10^{-3}$, followed by cosine decay \cite{loshchilov2016sgdr} over 50{,}000 iterations.

\vspace{1.5mm}

\textbf{Setting of TFDM.}
For each regime, the reverse-time ODE in \eqref{ODE} is solved using the Euler method with a uniform time discretization consisting of \(K = 2000\) steps. A one-hidden-layer neural network with 50 neurons and \(\tanh\) activation is used to parameterize \(\mathcal{N}_{\theta_1}\). The network is trained for 2000 iterations using the Adam optimizer with a learning rate of \(1.0 \times 10^{-2}\) and a weight decay of \(1.0 \times 10^{-6}\) to mitigate overfitting.

\vspace{1.5mm}

\textbf{Setting of SRAN.}
For each regime, the network architecture and training procedure are identical to those used in TFDM. In addition, the regularization coefficient is set as~\( \lambda = 0.1\).

\vspace{1.5mm}

\textbf{Setting of VAE.}
For each regime, we adopt a VAE architecture with a one-hidden-layer network of 50 neurons and a latent dimension of 9 (corresponding to three latent variables per dimension in \eqref{Triad}). The training loss is defined as a weighted objective with coefficients \(\alpha_{\text{mean}} = 0.1\), \(\alpha_{\text{var}} = 1.0\), and \(\beta_{\text{KL}} = 10^{-3}\). 

\vspace{1.5mm}

\textbf{Code Availability.} All numerical results presented in this work are fully reproducible using the publicly available repository at
\url{https://github.com/xxjan719/TurbulentFEX}.

\subsection{Symbolic Expression Performance in   Stage  I}
We evaluate the performance of Stage I across different turbulent regimes using FEX and Weak SINDy~\cite{messenger2021weak}, as summarized in Table~\ref{tab:multi_regime_expression}. FEX accurately recovers the governing equations across all regimes, including nonlinear interactions and forcing expressions with energy-conserving coefficients.  Under strong stochastic fluctuations, it continues to identify the correct dynamics, and in the random cascade regime, it further captures the deterministic component of \(\mathbf{F}(t)\), demonstrating robustness to regime-dependent forcing. The coefficient estimation error of \(\mathbf{B}(\mathbf{u},\mathbf{u})\) is further provided in Section~SM3.1 of the Supplementary Material.
In contrast, Weak SINDy typically approximates the forcing term polynomially rather than learning its explicit form, and inaccurate coefficient estimation often leads to unstable predictions. Its performance deteriorates further under strong stochastic influence, where it fails to identify interaction terms correctly. Classical SINDy exhibits similar but more pronounced limitations and is therefore omitted for brevity.

\begin{table*}[!htbp]
\caption{Recovered deterministic vector fields across turbulence regimes by FEX and Weak SINDy. Coefficients are rounded to three decimals and negligible terms are omitted for readability.}
\label{tab:multi_regime_expression}
\centering
\scriptsize
\setlength{\tabcolsep}{2.5pt}
\renewcommand{\arraystretch}{1.1}
\resizebox{\textwidth}{!}{%
\begin{tabular}{c|c|p{6cm}|p{6cm}|p{3.8cm}}
\toprule
\textbf{Regime} & \textbf{Dim} & \textbf{FEX Expression} & \textbf{Weak SINDy Expression} & \textbf{Ground Truth Expression} \\
\midrule

\multirow{3}{*}{\textbf{Equipartition}}
& 1 & $1.000u_2u_3-0.205u_1+0.997u_2-2.001u_3$
& $0.122u_2u_3- 0.869u_1+ 0.849u_2- 1.811u_3+0.518$
& $u_2u_3-0.2u_1+u_2-2u_3$ \\

& 2 & $-0.600u_1u_3-0.990u_1-0.105u_2-2.988u_3+0.014$
& $0.079u_1u_3- 1.448u_1- 0.381u_2-- 3.539u_3+0.0766$
& $-0.6u_1u_3-u_1-0.1u_2-3u_3$ \\

& 3 & $-0.402u_1u_2+1.997u_1+2.999u_2-0.104u_3-0.017$
& $ 0.169u_1u_2+ 0.027u_1+ 3.604u_2+0.614u_3  +1.023$
& $-0.4u_1u_2+2u_1+3u_2-0.1u_3$ \\

\midrule

\multirow{3}{*}{\textbf{Forward Cascade}}
& 1 & $2.035u_2u_3-0.986u_1+0.171u_2-0.149u_3-0.026$
& $8.555u_2u_3-0.833u_1-2.476u_3$
& $2u_2u_3-u_1$ \\

& 2 & $-1.001u_1u_3-2.002u_2$
& $-15.310u_1u_3-0.157u_1+0.418u_3$
& $-u_1u_3-2u_2$ \\

& 3 & $-0.999u_1u_2-1.994u_3$
& $7.365u_1u_2+0.176u_1-0.541u_2-0.910u_3$
& $-u_1u_2-2u_3$ \\

\midrule

\multirow{3}{*}{\textbf{Dual Cascade}}
& 1 & $1.985u_2u_3-0.995u_1+0.059u_2+0.039u_3+0.040$
& \(488.678u_2u_3- 2.087u_1+ 116.695u_2 + 117.007u_3 -4.750  \)
& $2u_2u_3-u_1+0.03u_2+0.06u_3$ \\

& 2 & $-1.000u_1u_3-0.030u_1-1.999u_2-0.088u_3-1.000$
& \(- 11.761u_1u_3- 2.315u_1+ 1.950u_2+ 3.021u_3-1.569    \)
& $-u_1u_3-0.03u_1-2u_2-0.09u_3-1$ \\

& 3 & $-1.000u_1u_2-0.060u_1+0.090u_2-1.999u_3+0.999$
& \(38.427u_1u_2+ 5.144u_1- 17.765u_2- 22.610u_3+3.844   \)
& $-u_1u_2-0.06u_1+0.09u_2-2u_3+1$ \\

\midrule

\multirow{3}{*}{\textbf{Periodic Cascade}}
& 1 & $1.974u_2u_3-0.994u_1+0.119u_2-0.098u_3+\sin(0.785t-0.091)-0.010$
& \(0.564u_2u_3- 0.247u_1 - 0.097u_2 - 0.117u_3 + 0.300t - 0.440t^2 + 0.091t^3 - 0.006t^4\)
& $2u_2u_3-u_1+\sin(\frac{\pi}{4}t)$ \\

& 2 & $-0.999u_1u_3-1.998u_2+\sin(0.785t+0.003)$
& \(- 0.623u_1u_3+ 0.6187u_1- 0.790u_2- 0.150u_3 + 0.741t - 0.427t^2 + 0.061t^3 - 0.002t^4 \)
& $-u_1u_3-2u_2+\sin(\frac{\pi}{4}t)$ \\

& 3 & $-1.000u_1u_2-1.998u_3+\sin(0.785t+0.001)$
& \(- 1.031u_1u_2+ 0.637u_1+ 0.092u_2- 0.436u_3 - 0.504t + 0.012t^3 \)
& $-u_1u_2-2u_3+\sin(\frac{\pi}{4}t)$ \\

\midrule

\multirow{3}{*}{\textbf{Random Cascade}}
& 1 & $2.089u_2u_3-1.019u_1+0.088u_2-0.015u_3+1.177\text{e}^{-4.947t}$
& \(4.464u_2u_3- 4.955u_1+ 15.238u_2 - 9.549u_3 + 3.850t - 1.306t^2 + 0.215t^3 - 0.017t^4\)
& $2u_2u_3-u_1+m(t)$ \\

& 2 & $-0.964u_1u_3-1.978u_2+0.064u_3+1.392\text{e}^{-4.998t}$
& \( 9.367u_1u_3- 4.520u_1 + 11.459u_2 - 6.821u_3  + 3.728t - 1.227t^2 + 0.196t^3 - 0.015t^4 \)
& $-u_1u_3-2u_2+m(t)$ \\

& 3 & $-0.975u_1u_2+0.044u_2-1.952u_3+1.438\text{e}^{-5.165t}$
& \(20.094u_1u_2- 4.626u_1+ 11.864u_2- 7.084u_3+3.663t -1.201t^2+0.191t^3 - 0.014t^4 \)
& $-u_1u_2-2u_3+m(t)$ \\

\bottomrule
\end{tabular}%
}
\end{table*}

\subsection{Tolerance of Stage I to Different Noise Levels}

To study the sensitivity of FEX to stochastic perturbations, we introduce a scaling parameter $c \in [0,2]$ for the white-noise forcing, replacing $\sigma_i\dot{W}_i$ with $c\sigma_i\dot{W}_i$ for $i=1,2,3$ allowing us to control the noise intensity from deterministic to strongly stochastic regimes. We focus on the equipartition, forward cascade, and dual cascade regimes as representative examples, where the same operator sequence applies across regimes and only the coefficients differ, and we assess the estimation error of the recovered interaction coefficients $\hat{B}_i$.
As shown in Figure~\ref{fig:discussion_quality_1}, the errors on a semi-log scale remain within
$10^{-6}$ to $10^{-2}$ across noise levels. This indicates that FEX maintains accurate coefficient recovery under stochastic perturbations while preserving the underlying energy conservation.

\begin{figure}[!htbp]
\begin{center}
    \includegraphics[width=1\linewidth]{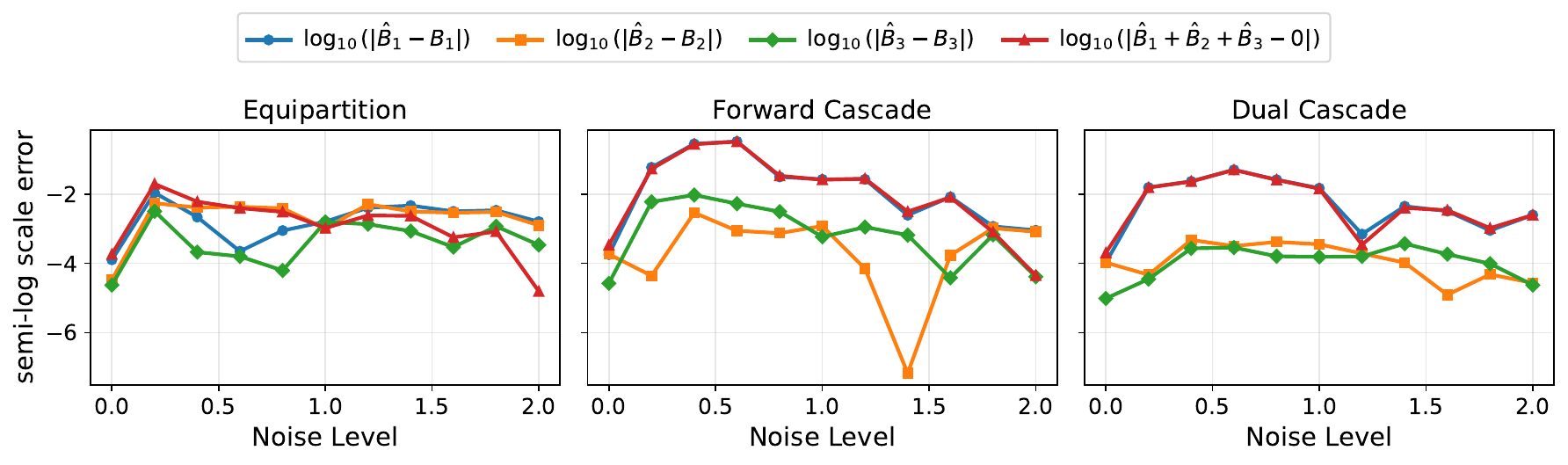}
    \caption{Coefficient performance of the nonlinear interaction terms across three different regimes under varying noise levels, evaluated using the same operator sequence within each regime. Here, \(B_i\) denotes the coefficient of the nonlinear term, and \(\hat{B}_i\) denotes its predicted value.}
\label{fig:discussion_quality_1}
\end{center}
\end{figure}

\subsection{Coefficient Estimation Improvement in  Stage I}

We now examine how the accuracy of coefficient estimation  by FEX improves with augmented data. Using the same three regimes and a fixed operator sequence in Section 6.3, we vary the number of training samples from \(1{,}000\) to \(10{,}000\) under standard white-noise forcing. The results in Figure~\ref{fig:discussion_quality_2} show that the estimation error of the nonlinear interaction coefficients decreases as the sample size increases, with the most significant improvement occurring between \(1{,}000\) and \(4{,}000\), particularly in the forward and dual cascade regimes. Beyond this range, further error reduction becomes limited, likely due to stochastic variability in the data. In contrast, the Equipartition regime exhibits weaker improvement, indicating greater sensitivity to noise. Overall, increasing the sample size improves coefficient estimation up to a certain threshold, after which the benefit diminishes.

\begin{figure}[!htbp]
\begin{center}
    \includegraphics[width=1\linewidth]{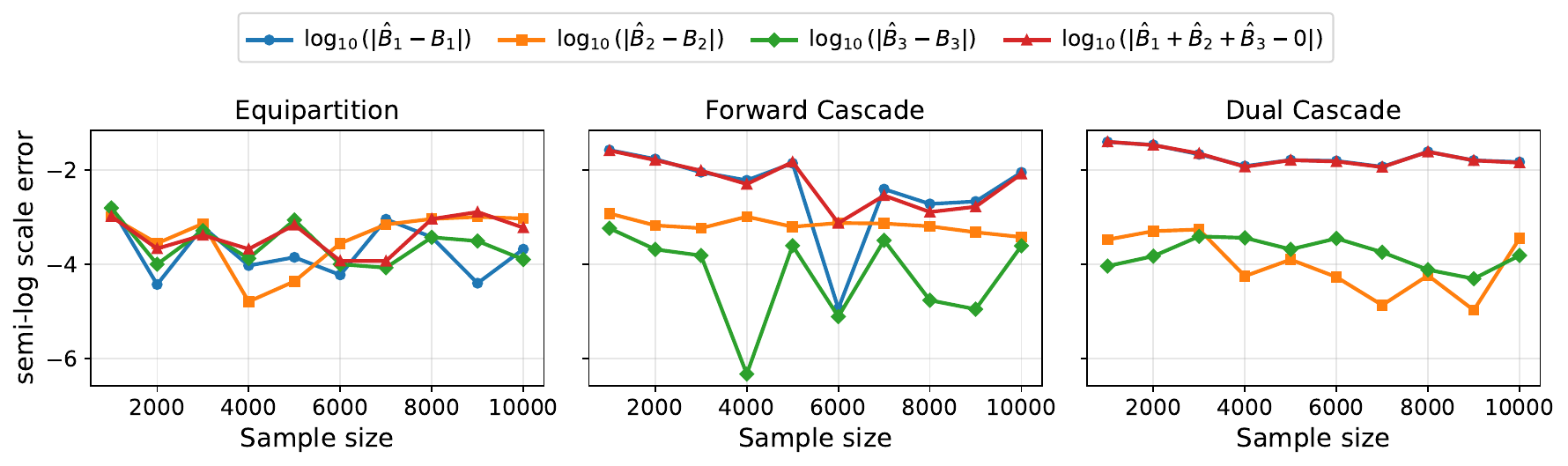}
    \caption{Coefficient performance of the nonlinear interaction terms across three different regimes under different training samples, evaluated using the same operator sequence within each regime. In dual cascade regime, red curve  overlaps with the  blue curve performance. }
\label{fig:discussion_quality_2}
\end{center}
\end{figure}

\subsection{Different moments Comparison}

We evaluate the two-stage framework by computing statistical moments up to order 5, which provide a sensitive measure of higher-order behavior. The results are shown in Figure~\ref{fig:discussion_moment_less_equal_3}--\ref{fig:discussion_moment_equal_45}, with additional density functions and moment results for different regimes provided in Sections~SM3.2--SM3.4 of the Supplementary Material.
Overall, all three methods capture the qualitative behavior across all regimes. SRAN achieves the closest agreement with the ground truth, particularly in the equipartition and dual cascade regimes, while TFDM shows only minor deviations. VAE preserves overall trends but exhibits larger amplitude fluctuations at higher orders. In the random cascade regime, performance becomes increasingly sensitive to the strength of the time-dependent forcing in \(\mathbf{F}(t)\), where strong stochastic effects make accurate recovery more challenging. These results demonstrate that the proposed framework remains effective under complex forcing while also highlighting its limitations under strong stochastic influence. 

\begin{figure}[!htbp]
\begin{center}
    \includegraphics[width=1.\linewidth]{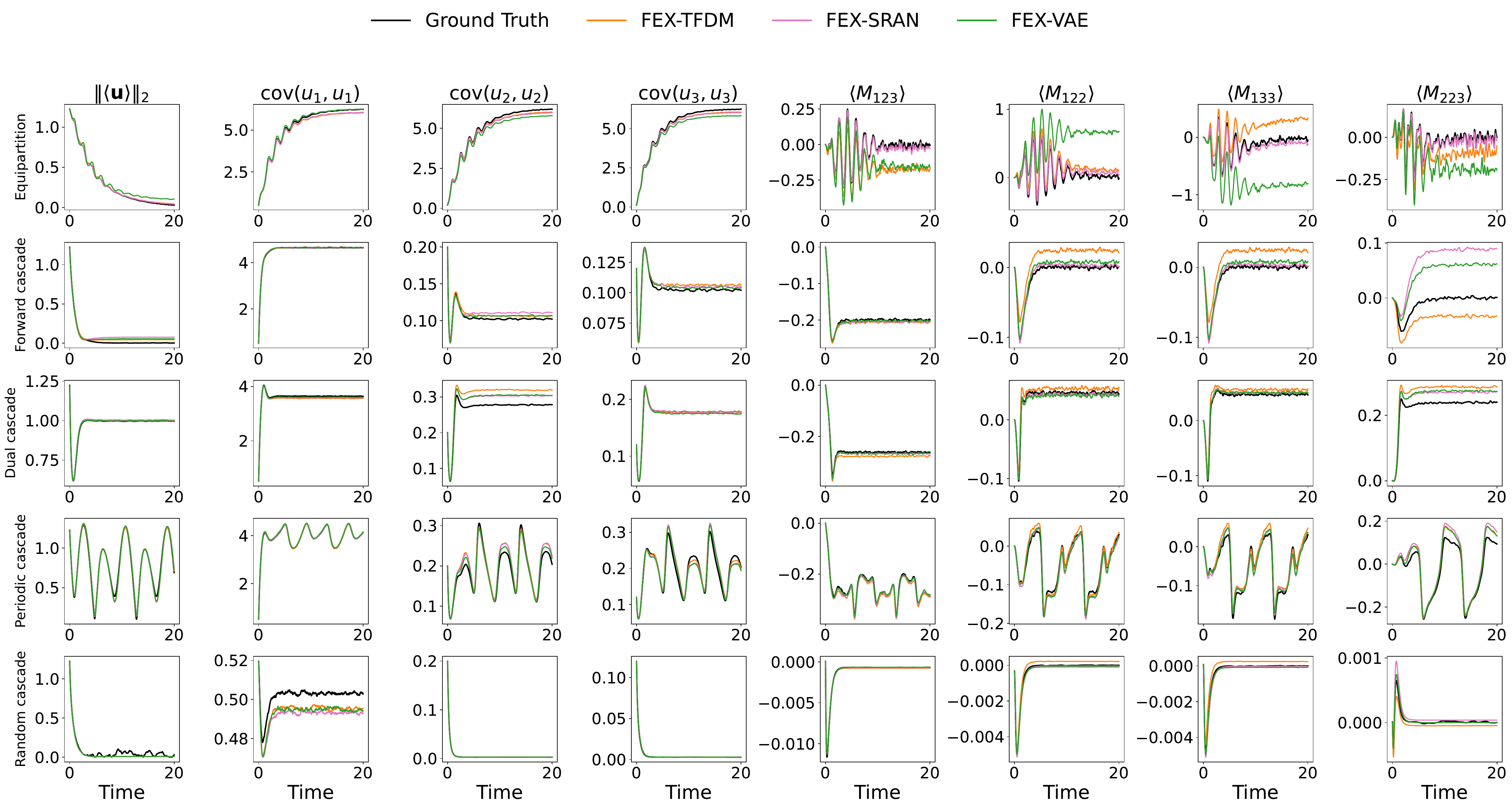}
\caption{Comparison of predicted moments up to third order across all regimes. 
\(\|\langle \mathbf{u} \rangle\|_2\) denotes the Euclidean norm of the mean state of \(\mathbf{u}\), and \(M_{ijk} = \mathbb{E}[u_i u_j u_k]\) denotes third-order mixed moments; for example, \(M_{122} = \mathbb{E}[u_1 u_2^2]\), \(M_{133} = \mathbb{E}[u_1 u_3^2]\), \(M_{233} = \mathbb{E}[u_2 u_3^2]\), and \(M_{123} = \mathbb{E}[u_1 u_2 u_3]\).
The orange curve (FEX-TFDM) corresponds to the two-stage framework using FEX combined with TFDM. 
The pink curve (FEX-SRAN) represents FEX combined with SRAN, while the green curve (FEX-VAE) denotes FEX combined with VAE-based activation.}
\label{fig:discussion_moment_less_equal_3}
\end{center}
\end{figure}

\begin{figure}[!htbp]
\begin{center}
    \includegraphics[width=1.\linewidth]{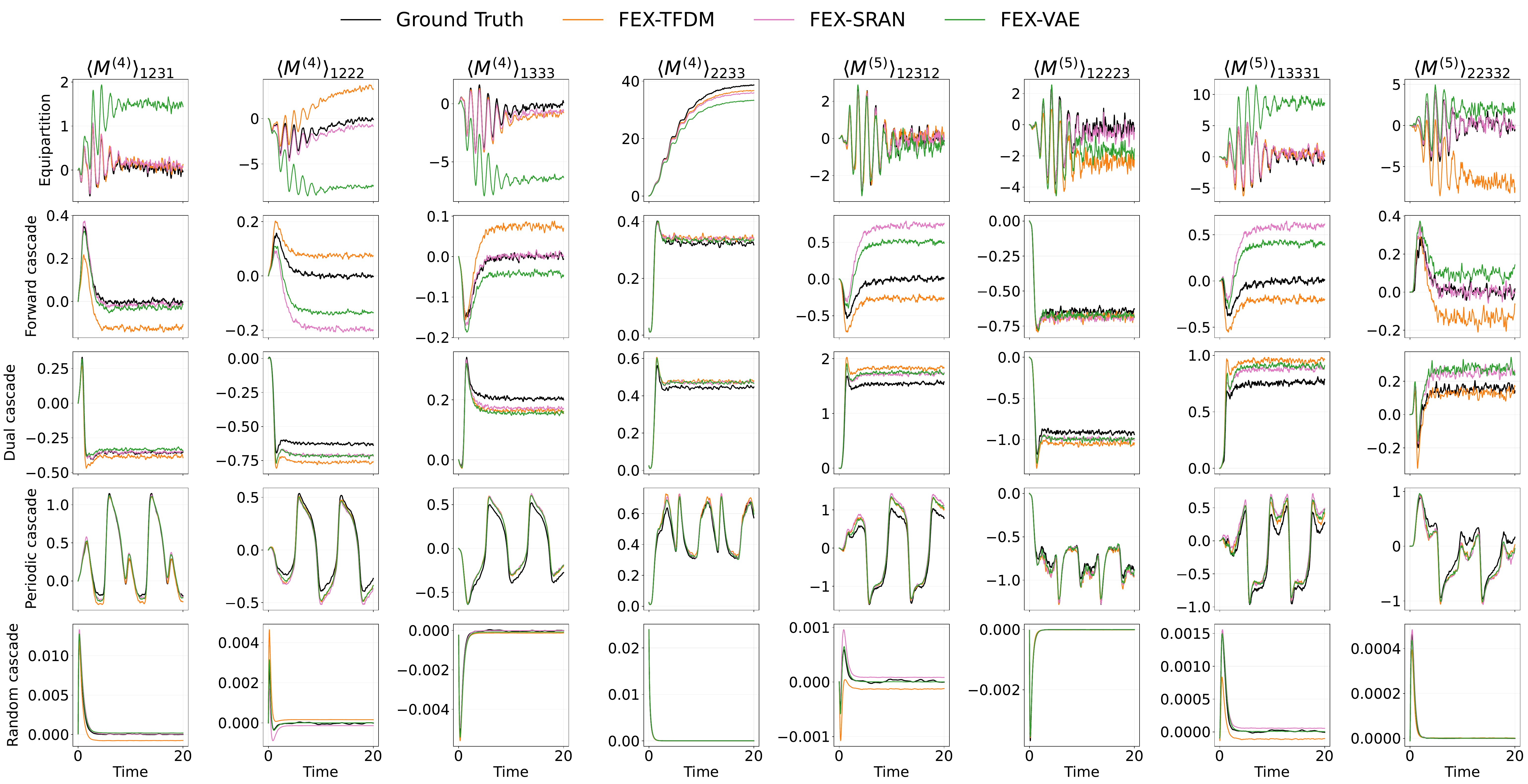}
\caption{Comparison of predicted fourth- and fifth-order moments across all regimes. 
 $M^{(k)}$ denotes moments of order $k$. For fourth-order moments, $M^{(4)}_{ijkl} = \mathbb{E}[u_i u_j u_k u_l]$, and for fifth-order moments, $M^{(5)}_{ijklm} = \mathbb{E}[u_i u_j u_k u_l u_m]$, where the indices specify the corresponding components.}
\label{fig:discussion_moment_equal_45}
\end{center}
\end{figure}



\subsection{Further Comparison for TFDM: Time-Independent vs. Time-Dependent Modeling}
For \eqref{Triad} driven by constant white-noise forcing, two training strategies can be adopted: the time-independent approach assumes a fixed stochastic distribution across all time steps, while the time-dependent approach allows the distribution to be updated at each time step. To investigate the differences between these strategies, we use TFDM as a representative example and evaluate the covariance structures of the state variables in regimes without forcing influence; the corresponding pseudo-algorithm is provided in Algorithm~SM2.3 in the Supplementary Material.
As shown in Figure~\ref{fig:discussion_quality_5}, the time-dependent approach more accurately captures short-term dynamics, particularly the oscillatory behavior in the covariance, with clear advantages in the cascade regime. However, beyond \(T_{\text{test}} = 10\), it becomes unstable due to ill-conditioned  and predictions deteriorate. The time-independent approach produces smoother and more stable long-term trajectories at the cost of a systematic bias and loss of fine-scale variations. These results suggest that the choice of modeling strategy involves a trade-off between short-term accuracy and long-term stability, which may be guided by the target time horizon in practice.

\begin{figure}[!htbp]
\begin{center}
    \includegraphics[width=0.7\linewidth]{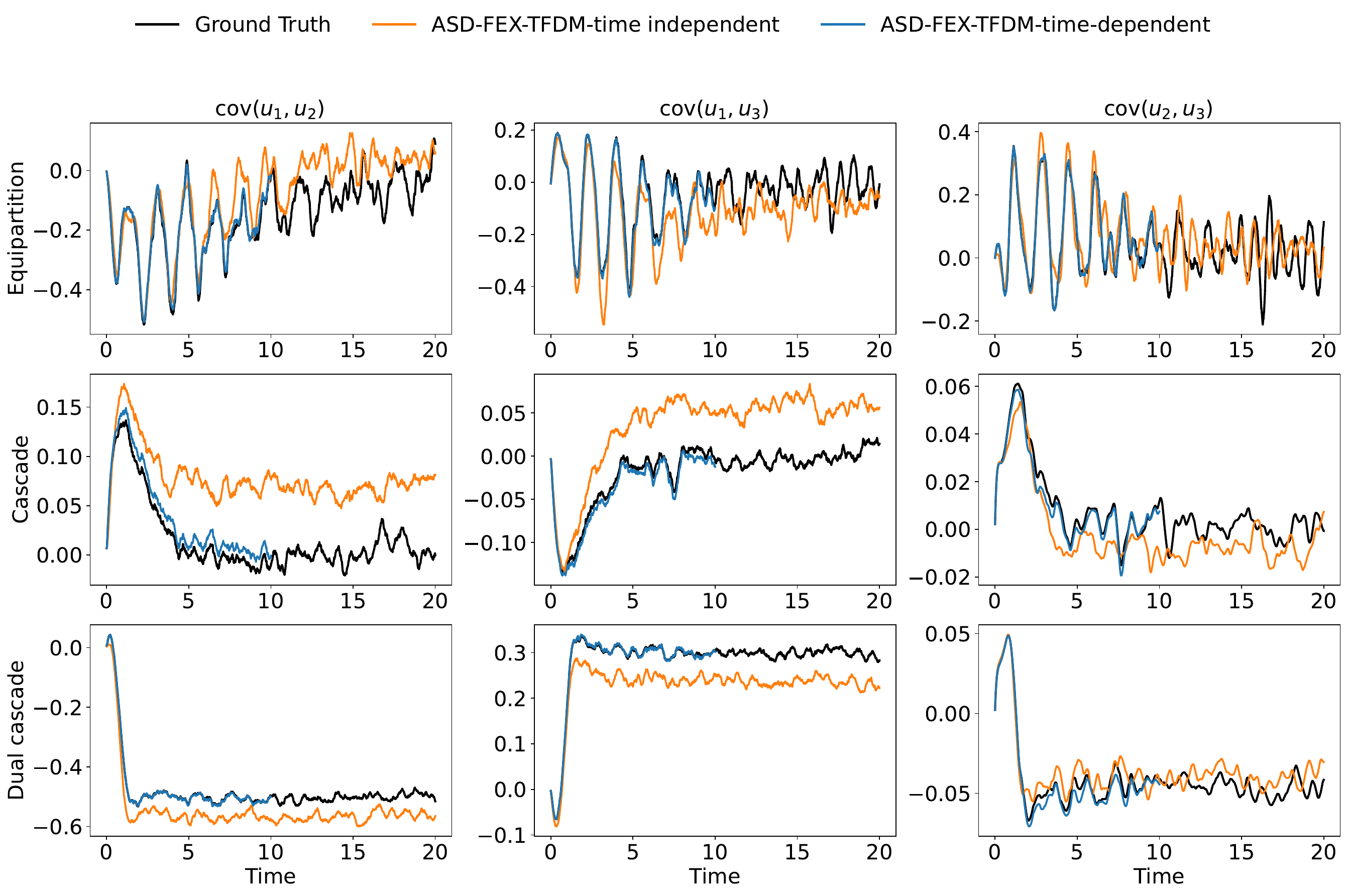}
    \caption{Comparison of predicted covariance structures across three regimes: equipartition, forward cascade, and dual cascade. Orange (FEX-TFDM, time-independent) and blue (FEX-TFDM, time-dependent) denote FEX-based symbolic models coupled with time-independent and time-dependent TFDM activation, respectively.}
\label{fig:discussion_quality_5}
\end{center}
\end{figure}

\section{Conclusion}\label{sec:sum}
In this paper, we extend symbolic discovery to identify turbulent dynamical systems and recover higher-order statistical behavior through a two-stage framework. In Stage I, FEX serves as the core symbolic learner, identifying closed-form representations of the underlying dynamics. It accurately recovers nonlinear interaction terms and preserves intrinsic structures, such as conservation laws, while remaining insensitive to stochastic perturbations. In Stage II, generative models, including TFDM, SRAN, and VAE, are used to model and resample the stochastic components, enabling predictive dynamics through enriched sample generation. This stage further facilitates accurate recovery of higher-order statistical quantities, including moments, without being restricted to a fixed prediction horizon. 
Numerical experiments on the triad system validate these properties, demonstrating that the proposed framework reliably captures turbulent dynamics across multiple regimes. Furthermore, it remains effective in more challenging settings with multiple sources of stochasticity, such as the random cascade regime, highlighting its applicability to complex stochastic systems.

In future work, we aim to extend this framework in several directions. First, we plan to extend the framework to model multiscale coupled stochastic effects, including white-noise-driven processes in both space and time, with the goal of accurately predicting long-term high-order moment behavior while maintaining the ability to recover interpretable nonlinear interactions and forcing terms without prior assumptions, as well as to stochastic PDE settings where identifying high-order moment dynamics remains a significant challenge. Second, the computational cost of TFDM scales with the number of time steps and data volume, and developing more scalable implementations through adaptive step size selection or more efficient sampling strategies is an important practical issue. Finally, a more systematic component-wise symbolic discovery strategy that further exploits the interaction structure of high-dimensional systems will also be explored.

\bibliographystyle{plain}
\bibliography{main/main}

\end{document}


\maketitle
\section{Proof Details}
\subsection{First Stage for FEX}
\subsubsection{Proof for Theorem 5.2}

\begin{proof}
Stacking all $MN$ samples gives
\[
Y^i = \Gamma^i \alpha^i + r^i + \xi^i,
\]
where $Y^i \in \mathbb{R}^{MN}$ is the response vector with entries $Y^i_{(m,n)} = \Delta u_i^{(m)}(t_n)/h$, $\Gamma^i \in \mathbb{R}^{MN \times K}$ is the design matrix with rows $(\Psi^i_{m,n})^\top$, and $r^i, \xi^i \in \mathbb{R}^{MN}$ are the stacked approximation residual and stochastic regression error vectors, respectively. The least-squares estimator satisfies the normal equation, so
\[
\hat{\alpha}^i - \alpha^i
=
\left(\frac{1}{MN}(\Gamma^i)^\top \Gamma^i\right)^{-1}
\left(\frac{1}{MN}(\Gamma^i)^\top (r^i + \xi^i)\right).
\]
By Assumptions~5.1 and~5.3, ergodicity and linear independence give
\[
\frac{1}{MN}(\Gamma^i)^\top \Gamma^i \xrightarrow{P} G_i, 
\]
where $G_i$ is positive definite. By continuity of matrix inversion at nonsingular matrices,
\[
\left(\frac{1}{MN}(\Gamma^i)^\top \Gamma^i\right)^{-1} = G_i^{-1} + o(1),
\]
which yields the representation
\[
\hat{\alpha}^i - \alpha^i
=
\left(G_i^{-1} + o(1)\right)
\left(\frac{1}{MN}(\Gamma^i)^\top (r^i + \xi^i)\right).
\]

By Assumption~5.4, the polynomial basis can be chosen so that
\[
|r_i(t_n)| \leq \sup_{t \in [0,T]} |r_i(t)| \leq \varepsilon_i, \qquad \text{for all } n = 0, \dots, N-1.
\]
Since the basis functions are square-integrable under $\rho$, there exists a constant $C_\psi < \infty$ such that $\|\Psi^i_{m,n}\| \leq C_\psi$ for all $m, n$. Since $r^i$ is deterministic once the basis is fixed,
\[
\left\|\mathbb{E}\left[\frac{1}{MN}(\Gamma^i)^\top r^i\right]\right\|
=
\left\|\frac{1}{MN}(\Gamma^i)^\top r^i\right\|
\leq
\frac{1}{MN} \sum_{m=1}^{M} \sum_{n=0}^{N-1} \|\Psi^i_{m,n}\| \cdot |r_i(t_n)|
\leq
C_\psi \varepsilon_i.
\]

From \eqref{FM1} and (5.3), the normalized regression error satisfies
\[
\xi_i^{(m)}(t_n)
=
\frac{\Delta u_i^{(m)}(t_n)}{h} - C_i(\mathbf{u}^{(m)}(t_n), t_n)
=
\frac{S_i(t_n, z; h)}{h}.
\]
By Assumption~(5.2), the stochastic increment admits the decomposition
\[
\frac{S_i(t_n, z; h)}{h}
=
\frac{1}{\sqrt{h}} \sum_{k=1}^{d} \sigma_{ik}(\mathbf{u}(t_n), t_n) z_k
+
\frac{R_i(t_n, h)}{h}.
\]
Taking the conditional expectation given $\mathbf{u}(t_n)$,
\[
\mathbb{E}\!\left[\xi_i^{(m)}(t_n) \mid \mathbf{u}(t_n)\right]
=
\frac{1}{\sqrt{h}} \sum_{k=1}^{d} \sigma_{ik}(\mathbf{u}(t_n), t_n) \mathbb{E}[z_k \mid \mathbf{u}(t_n)]
+
\frac{1}{h}\mathbb{E}[R_i(t_n, h) \mid \mathbf{u}(t_n)]
=
0,
\]
where the first term vanishes because $z_k \sim \mathcal{N}(0,1)$ is independent of $\mathbf{u}(t_n)$, and the second term vanishes exactly because $R_i$ is an It\^{o} integral. By the tower property,
\[
\mathbb{E}\!\left[\frac{1}{MN}(\Gamma^i)^\top \xi^i\right]
=
\frac{1}{MN} \sum_{m=1}^{M} \sum_{n=0}^{N-1}
\mathbb{E}\!\left[\Psi^i_{m,n} \cdot \mathbb{E}\!\left[\xi_i^{(m)}(t_n) \mid \mathbf{u}^{(m)}(t_n)\right]\right]
=
0.
\]

Combining above arguments,
\[
\left\|\mathbb{E}\!\left[\frac{1}{MN}(\Gamma^i)^\top (r^i + \xi^i)\right]\right\|
\leq
C_\psi \varepsilon_i.
\]
Using the representation from Step 1 and boundedness of $G_i^{-1} + o(1)$,
\[
\left\|\mathbb{E}[\hat{\alpha}^i - \alpha^i]\right\|
\leq
C_i \left(\varepsilon_i + o(1)\right),
\]
where $C_i > 0$ depends only on $G_i$ and the basis.

For the variance bound,
since $r^i$ is deterministic, the variance comes entirely from $\xi^i$:
\[
\mathrm{Var}(\hat{\alpha}^i)
\leq
\|G_i^{-1} + o(1)\|^2 \cdot \mathbb{E}\!\left[\left\|\frac{1}{MN}(\Gamma^i)^\top \xi^i\right\|^2\right].
\]
By the Cauchy-Schwarz inequality,
\[
\left\|\frac{1}{MN} \sum_{m,n} \Psi^i_{m,n} \xi_i^{(m)}(t_n)\right\|^2
\leq
\frac{1}{MN} \sum_{m=1}^{M} \sum_{n=0}^{N-1} \|\Psi^i_{m,n}\|^2 \cdot |\xi_i^{(m)}(t_n)|^2.
\]
Taking expectations and using the conditional second moment,
\[
\mathbb{E}\!\left[|\xi_i^{(m)}(t_n)|^2 \mid \mathbf{u}(t_n)\right]
=
\mathbb{E}\!\left[\left|\frac{S_i(t_n,z;h)}{h}\right|^2 \mid \mathbf{u}(t_n)\right].
\]
Expanding using the decomposition in Assumption~5.2,
\[
\mathbb{E}\!\left[\left|\frac{S_i(t_n,z;h)}{h}\right|^2 \mid \mathbf{u}(t_n)\right]
=
\frac{1}{h^2}\left[
h \sum_{k=1}^{d} \sigma_{ik}^2(\mathbf{u}(t_n),t_n)
+
\mathbb{E}\!\left[|R_i(t_n,h)|^2 \mid \mathbf{u}(t_n)\right]
\right]\]
\[=
\frac{1}{h}\sum_{k=1}^{d}\sigma_{ik}^2 + O(1)
=
O\!\left(\frac{1}{h}\right).
\]
Together with $\|\Psi^i_{m,n}\|^2 \leq C_\psi^2$ and ergodicity,
\[
\mathbb{E}\!\left[\left\|\frac{1}{MN}(\Gamma^i)^\top \xi^i\right\|^2\right]
\leq
\frac{1}{MN} \cdot C_\psi^2 \cdot O\!\left(\frac{1}{h}\right)
=
O\!\left(\frac{1}{MNh}\right)
=
O\!\left(\frac{1}{MT}\right),
\]
where we used $T = Nh$. For fixed $M$, multiplying by the bounded factor $\|G_i^{-1} + o(1)\|^2$,
\[
\mathrm{Var}(\hat{\alpha}^i)
\leq
\frac{\widetilde{C}_i}{T}(1 + o(1)),
\]
where $\widetilde{C}_i > 0$ depends only on $G_i$, the basis, and $\sigma$. This completes the proof.
\end{proof}
\subsection{Second Stage for TFDM}

\subsubsection{Proof of Lemma 5.3}
\begin{proof}
By \eqref{FM1} and \eqref{residual},
\[
\mathbf R_n
=
h\bigl(C(u(t_n),t_n)-C^*(u(t_n))\bigr)+\mathbf S(t_n,z;h).
\]
Therefore, by the triangle inequality,
\[
\|\mathbf R_n\|_2
\le
h\,\|C(u(t_n),t_n)-C^*(u(t_n))\|_2
+
\|\mathbf S(t_n,z;h)\|_2.
\]

Under Assumptions~5.1--5.5, the state process remains in a bounded region, the symbolic approximation error is bounded, and the stochastic term is bounded on the training set. Thus there exist constants \(M_C,M_S>0\), independent of \(n\), such that
\[
\|C(u(t_n),t_n)-C^*(u(t_n))\|_2\le M_C,
\qquad
\|\mathbf S(t_n,z;h)\|_2\le M_S,
\]
for all \(n=0,\dots,N-1\). It follows that
\[
\|\mathbf R_n\|_2
\le
hM_C+M_S
=: M_R,
\qquad n=0,\dots,N-1.
\]
Hence there exists a constant \(M_R<\infty\) such that
\[
\|\mathbf R_n\|_2\le M_R,\qquad n=0,\dots,N-1.
\]

Consequently, every residual sample lies in the closed \(\ell_2\)-ball of radius \(M_R\), so the residual training dataset
\(
\mathbf R=[\mathbf R_0,\dots,\mathbf R_{N-1}]^\top
\)
is contained in an \(\ell_2\)-ball of radius \(M_R\). See \cite{wang2026error,wanner2024higher} for related proofs. 
\end{proof}

\subsubsection{Proof of Theorem 5.4}
\begin{proof}
We decompose the total error into three parts:
\[
\mathbf y^{*}(\boldsymbol{\xi}_1)-\mathcal N_{\theta_1}(\boldsymbol{\xi}_1)
=
\bigl(\mathbf y^{*}(\boldsymbol{\xi}_1)-\mathbf y(\boldsymbol{\xi}_1)\bigr)
+
\bigl(\mathbf y(\boldsymbol{\xi}_1)-\mathbf y^K(\boldsymbol{\xi}_1)\bigr)
+
\bigl(\mathbf y^K(\boldsymbol{\xi}_1)-\mathcal N_{\theta_1}(\boldsymbol{\xi}_1)\bigr),
\]
where \(\mathbf y(\boldsymbol{\xi}_1)\) denotes the discretized exact solution of \eqref{ODE}, and \(\mathbf y^K(\boldsymbol{\xi}_1) = \mathbf{y}\) denotes its numerical approximation.

Applying the triangle inequality and taking expectation with respect to \(\boldsymbol{\xi}_1\), we obtain
\begin{equation}\label{eq:three_part_decomp}
\mathbb E_{\boldsymbol{\xi}_1}\!\left[
\|\mathbf y^{*}(\boldsymbol{\xi}_1)-\mathcal N_{\theta_1}(\boldsymbol{\xi}_1)\|_2
\right]
\le
I_1+I_2+I_3,
\end{equation}
where
\[
I_1=\mathbb E_{\boldsymbol{\xi}_1}\!\left[
\|\mathbf y^{*}(\boldsymbol{\xi}_1)-\mathbf y(\boldsymbol{\xi}_1)\|_2
\right], \quad I_2=\mathbb E_{\boldsymbol{\xi}_1}\!\left[
\|\mathbf y(\boldsymbol{\xi}_1)-\mathbf y^K(\boldsymbol{\xi}_1)\|_2
\right],
\]
\[I_3=\mathbb E_{\boldsymbol{\xi}_1}\!\left[
\|\mathbf y^K(\boldsymbol{\xi}_1)-\mathcal N_{\theta_1}(\boldsymbol{\xi}_1)\|_2
\right].\]

We bound each term separately.

For \(I_1\), Assumption~5.5 implies that the initialization corresponds to a Gaussian smoothing with covariance \(\Sigma=C_0^2\mathbf I_d\). Therefore,
\[
I_1
\le
\sqrt{\operatorname{tr}(\Sigma)}
=
C_0\sqrt d.
\]

For \(I_2\), the global discretization error bound in Theorem~3.8 of \cite{wang2026error} yields
\[
I_2
=
\mathbb E_{\boldsymbol{\xi}_1}\!\left[
\|\mathbf y(\boldsymbol{\xi}_1)-\mathbf y^K(\boldsymbol{\xi}_1)\|_2
\right]
\le
C_1dh,
\]
for some constant \(C_1>0\).

For \(I_3\), by the triangle inequality and the assumption on \(\mathcal N_{\theta_1}\),
\[
I_3
\le
\mathbb E_{\boldsymbol{\xi}_1}\!\left[\|\mathbf y^K(\boldsymbol{\xi}_1)\|_2\right]
+
\mathbb E_{\boldsymbol{\xi}_1}\!\left[\|\mathcal N_{\theta_1}(\boldsymbol{\xi}_1)\|_2\right]
\le
\widetilde C + C_2
=: C_3,
\]
where \(\widetilde C>0\) follows from the boundedness of the exact solution together with the stability of the numerical scheme.

Substituting these three estimates into \eqref{eq:three_part_decomp}, we conclude that
\[
\mathbb E_{\boldsymbol{\xi}_1}\!\left[
\|\mathbf y^{*}(\boldsymbol{\xi}_1)-\mathcal N_{\theta_1}(\boldsymbol{\xi}_1)\|_2
\right]
\le
C_0\sqrt d + C_1dh + {C_3}.
\]
This completes the proof.
\end{proof}




\section{Pseudo Algorithms}

Following the implementations described in the main paper, we present the workflow in the following algorithms. 
Algorithm~\ref{alg:training_ASD_FEX_SRAN} and Algorithm~\ref{alg:training_ASD_FEX_VAE} correspond to the use of SRAN and VAE models as activation methods, respectively. 
Algorithm~\ref{alg:training_ASD_FEX_TFDM_dependent} provides the detailed implementation of the time-dependent TFDM described in Section~5.6.

\begin{algorithm}[!htbp]
\caption{Workflow of two-stage framework by FEX and SRAN}
\label{alg:training_ASD_FEX_SRAN}

\textbf{Input:} Observed state trajectories 
$\{\mathbf{u}^{(m)}(t_n)\}_{m,n=1}^{M,N}$,
where $\mathbf{u}^{(m)}(t_n)\in\mathbb{R}^d$.\\
\textbf{Output:} Estimated symbolic model $C^{*}$; estimated network model $\mathcal{N}_{\theta_2^{*}}(\cdot)$ with the optimal parameter \(\theta_2^{*}\).

\begin{algorithmic}[1]
\State \textbf{Procedure:}

\For{$i = 1, \dots, d$}
    \State Minimize $\mathcal{L}(\Theta_i)$ via~\eqref{CO} following Section~\ref{FEX_search}.
\EndFor

\State Construct $\mathbf{C}^{*}$; then compute the residual \(\mathbf{R}\) via~(2.2).

\State Set $\boldsymbol{\xi}_2 \sim \mathcal{N}(0,\mathbf{I}_d)$ as input data and construct the training pair $(\boldsymbol{\xi}_2,\mathbf{R})$.
    
\State Train the neural network $\mathcal{N}_{\theta_2}$ using the training pair $(\boldsymbol{\xi}_2,\mathbf{R})$ by optimizing (2.8).

\end{algorithmic}
\end{algorithm}

\begin{algorithm}[!htbp]
\caption{Workflow of two-stage framework by FEX with VAE}
\label{alg:training_ASD_FEX_VAE}

\textbf{Input:} Observed state trajectories 
$\{\{\mathbf{u}^{(m)}(t_n)\}_{m,n=1}^{M,N}$,
where $\mathbf{u}^{(m)}(t_n)\in\mathbb{R}^d$; hyperparameters 
\(\alpha_{\text{mean}}, \alpha_{\text{var}}, \beta_{\text{KL}}\).\\
\textbf{Output:} Estimated symbolic model $\mathbf{C}^{*}$; trained decoder $\mathcal{N}_{\theta_3^{*}}(\cdot)$ with optimal parameters \(\theta_3^{*}\).

\begin{algorithmic}[1]
\State \textbf{Procedure:}

\For{$i = 1, \dots, d$}
    \State Minimize $\mathcal{L}(\Theta_i)$ via~\eqref{CO} as described in Section~\ref{FEX_search}.
\EndFor

\State Construct the symbolic model $\mathbf{C}^{*}$ and compute the residual $\mathbf{R}$ via~\eqref{final_expression}.

\State Define the encoder $\mathcal{N}_{\phi_3}(\mathbf{z}\mid \mathbf{R})$ and decoder $\mathcal{N}_{\theta_3}(\mathbf{R}\mid \mathbf{z})$, where the encoder produces $\mu_{\phi_3}(\mathbf{R})$ and $\sigma_{\phi_3}(\mathbf{R})$, and sample
\[
\mathbf{z} = \mu_{\phi_3}(\mathbf{R}) + \sigma_{\phi_3}(\mathbf{R}) \odot \boldsymbol{\epsilon}, 
\quad \boldsymbol{\epsilon} \sim \mathcal{N}(0,\mathbf{I}_d).
\]

\State Train $(\phi_3,\theta_3)$ by minimizing (2.9) with given hyperparameters \(\alpha_{\text{mean}}, \alpha_{\text{var}}, \beta_{\text{KL}}\).

\end{algorithmic}
\end{algorithm}

\begin{algorithm}[!htbp]
\caption{Workflow of two-stage framework by FEX and time-dependent TFDM}
\label{alg:training_ASD_FEX_TFDM_dependent}

\textbf{Input:} Observed state trajectories 
$\{\mathbf{u}^{(m)}(t_n)\}_{m,n=1}^{M,N}$,
where $\mathbf{u}^{(m)}(t_n)\in\mathbb{R}^d$; index $K$.\\
\textbf{Output:} Estimated symbolic model $C^{*}$; estimated stochastic model $\mathcal{N}_{\theta_3^{*}}^{\text{all}}$ with the optimal parameters \(\theta_3^{*}\).

\begin{algorithmic}[1]
\State \textbf{Procedure:}

\For{$i = 1, \dots, d$}
    \State Minimize $\mathcal{L}(\Theta_i)$ via~\eqref{CO} following Section~\ref{FEX_search}.
\EndFor

\State Assemble the vector operator $\mathbf{C}^{*}$ and jointly fine-tune its coefficients using $\{\mathbf{u}^{(m)}(t_n)\}_{m,n=1}^{M,N}$.

\For{$n = 1,\dots,N$}
    \State Compute the residual $\mathbf{R}_n$ according to~(2.2).
    
    \State Perform FAISS-based nearest-neighbor selection to obtain a representative sample $Z_0$ from $\mathbf{R}_n$.
    
    \State Sample $Z_1 = z_n \sim \mathcal{N}(0,\mathbf{I}_d)$ and set $Z_{\tau_K} = Z_1$.
    
    \For{$k = K,\dots,1$}
        \State Compute the score function $V(Z_{\tau_k},\tau_k)$ using the Monte Carlo estimator in~\cite{liu2025training}.
        \State Solve the reverse dynamics by numerically integrating~\eqref{ODE} to obtain $Z_{\tau_{k-1}}$.
    \EndFor
    
    \State Set $y_n = Z_{\tau_0}$ and construct the training pair $(z_n,y_n)$.
    
    \State Train the network model $\mathcal{N}_{\theta_3}^n$ using the dataset $(z_n,y_n)$ by solving~\eqref{optim}.
\EndFor

\State Assemble the stochastic model 
$\mathcal{N}_{\theta_3^{*}}^{\text{all}} = \{\mathcal{N}_{\theta_3^{*}}^{n}(\cdot)\}_{n=1}^{N}$.
\end{algorithmic}
\end{algorithm}

\section{Additional Experimental Results}

In this section, we present additional results for the regimes in Section~\ref{Numerical_Experiment}. In particular, we include the comparison of deterministic expressions, evaluations of first-moment performance, energy evolution, and state density from similar initial conditions, following measurements similar to those in \cite{qi2017strategies}. These results provide further evidence in the triad turbulent system that the two-stage framework, with FEX and activated generative models, effectively captures the underlying dynamics.

\subsection{Expression Performance of Interaction term}
To quantify the accuracy of the discovered nonlinear interaction terms that govern energy transfer among modes in turbulent dynamics, Table~\ref{tab:error_nonlinear_interaction} reports the absolute and relative errors of the quadratic coefficients across all regimes by FEX. These recovered coefficients exhibit very small errors, typically on the order of  \(10^{-2} - 10^{-5}\), which indicates that the FEX reliably captures the essential nonlinear structure of the turbulent dynamics.

\begin{table}[!htbp]
\centering
\caption{Absolute and relative errors of recovered nonlinear interaction coefficients across different turbulence regimes by FEX.}
\label{tab:error_nonlinear_interaction}
\footnotesize
\setlength{\tabcolsep}{4pt}
\renewcommand{\arraystretch}{1.15}
\begin{tabular}{c c c c c}
\hline
\textbf{Regime} & \textbf{Term} & \textbf{True Coef} & \textbf{Abs Error} & \textbf{Rel Error} \\
\hline

\textbf{Equipartition} & $x_2x_3$ & $1$    & $7.5178\text{e}{-5}$ & $7.5178\text{e}{-5}$ \\
         & $x_1x_3$ & $-0.6$ & $3.2725\text{e}{-4}$ & $5.4541\text{e}{-4}$ \\
         & $x_1x_2$ & $-0.4$ & $1.5529\text{e}{-3}$ & $3.8823\text{e}{-3}$ \\

\hline

\textbf{Forward Cascade} &$x_2x_3$ & $2$  & $8.3042\text{e}{-3}$ & $4.1521\text{e}{-3}$ \\
                & $x_1x_3$ & $-1$ & $3.3846\textbf{e}{-4}$ & $3.3846\textbf{e}{-4}$ \\
                & $x_1x_2$ & $-1$ & $2.4628\textbf{e}{-4}$ & $2.4628\textbf{e}{-4}$ \\

\hline

\textbf{Dual Cascade}& $x_2x_3$ & $2$  & $1.4194\text{e}{-5}$ & $7.0968\text{e}{-6}$ \\
             & $x_1x_3$ & $-1$ & $4.2251\text{e}{-3}$ & $4.2251\text{e}{-3}$ \\
             & $x_1x_2$ & $-1$ & $7.8148\text{e}{-3}$ & $7.8148\text{e}{-3}$ \\

\hline

\textbf{Periodic Cascade} & $x_2x_3$ & $2$  & $2.5928\text{e}{-2}$ & $1.2964\text{e}{-2}$ \\
                 & $x_1x_3$ & $-1$ & $6.1374\text{e}{-4}$ & $6.1374\text{e}{-4}$ \\
                 & $x_1x_2$ & $-1$ & $4.9872\text{e}{-4}$ & $4.9872\text{e}{-4}$ \\

\hline

\textbf{Random Cascade} & $x_2x_3$ & $2$  & $8.8930\text{e}{-2}$ & $4.4465\text{e}{-2}$ \\
               & $x_1x_3$ & $-1$ & $3.6133\text{e}{-2}$ & $3.6133\text{e}{-2}$ \\
               & $x_1x_2$ & $-1$ & $2.5452\text{e}{-2}$ & $2.5452\text{e}{-2}$ \\
\hline

\hline
\end{tabular}
\end{table}

\subsection{First-moment Performance across Different Regimes}

In addition to the results presented in the main paper, we further examine the first-moment behavior for each component of the triad system, namely \(\langle u_1 \rangle\), \(\langle u_2 \rangle\), and \(\langle u_3 \rangle\). The results are shown in Figure~\ref{fig:first_moment}.
From the figure, we observe that the mean state dynamics predicted by the two-stage framework closely match the ground truth across all regimes. This agreement is satisfied even under logarithmic scaling, indicating that both large- and small-scale variations are accurately captured.
Since the activation step primarily models stochastic fluctuations and contributes less to first-order statistics, the first-moment performance mainly reflects the accuracy of the symbolic component. Therefore, this evaluation also serves as a diagnostic of the quality of the recovered deterministic structure.

\begin{figure}[!htbp]
    \centering
    \includegraphics[width=1\linewidth]{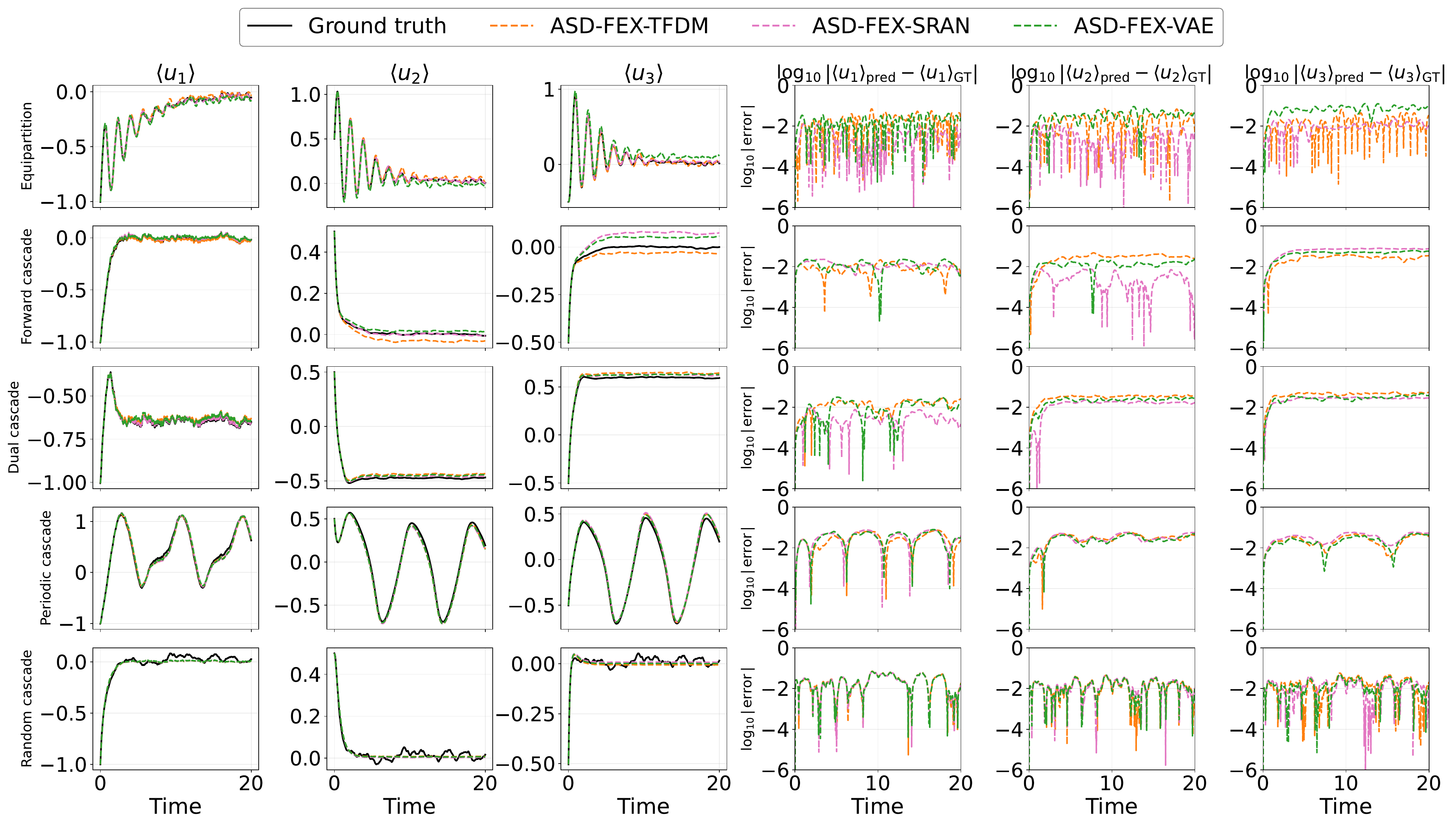}
    \caption{First-moment comparison of \(\langle u_1 \rangle\), \(\langle u_2 \rangle\), and \(\langle u_3 \rangle\) across different regimes. Here, \(\langle \cdot \rangle_{\mathrm{pred}}\) denotes the ASD prediction, while \(\langle \cdot \rangle\) denotes the ground truth. The close agreement demonstrates the accuracy of the recovered deterministic structure.The orange curve (FEX-TFDM) corresponds to FEX combined with the TFDM. The pink curve (FEX-SRAN) represents the FEX with SRAN, while the green curve (FEX-VAE) denotes the FEX with VAE.}
    \label{fig:first_moment}
\end{figure}

\subsection{Energy evolution across different regimes}

We next examine the framework's ability to reproduce the energy dynamics of each mode in the triad system. The energy evolution for each component is shown in Figure~\ref{fig:energy_evolution}.
From the figure, we observe that the predicted energy trajectories closely match the ground truth across all regimes. In particular, the two-stage framework accurately captures both the transient behavior and long-term energy evolution. Notably, this agreement persists beyond the training horizon (from \(T=10\) to \(T=20\)), demonstrating the strong generalization capability of the model. 
\begin{figure}[!htbp]
    \centering
    \includegraphics[width=1\linewidth]{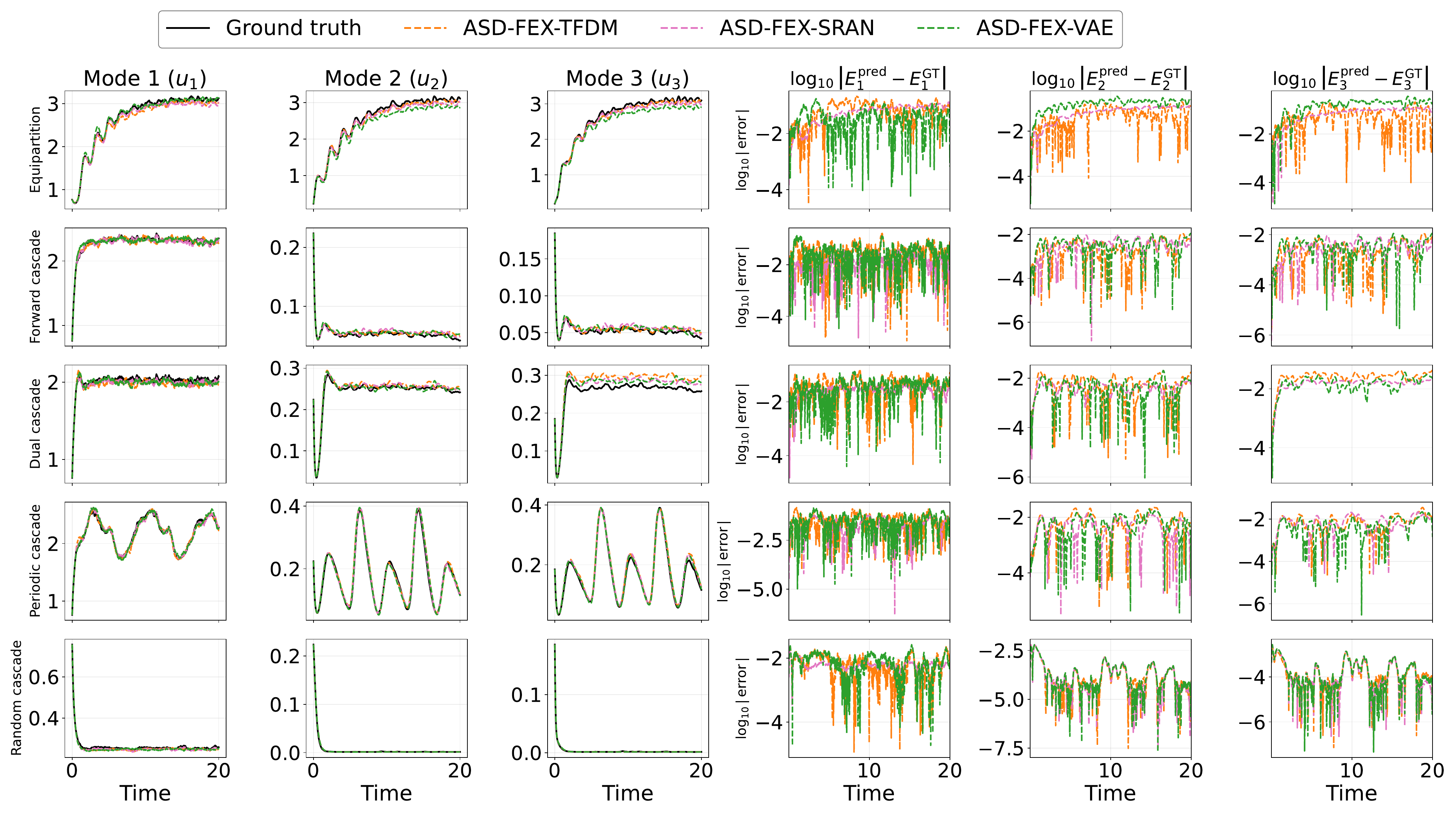}
    \caption{Energy evolution of each mode across different regimes. Here, \(E_i^{\mathrm{pred}}\) denotes the prediction for the energy of mode \(i=1,2,3\), while \(E_i^{\mathrm{GT}}\) denotes the corresponding ground truth. The orange curve (FEX-TFDM) corresponds to FEX combined with the TFDM . The pink curve (FEX-SRAN) represents FEX with SRAN, while the green curve (FEX-VAE) denotes FEX with VAE.}
    \label{fig:energy_evolution}
\end{figure}

\subsection{Higher-Order Statistics and Density Estimation}
We further evaluate all models on higher-order statistics up to order 7. Overall, FEX-SRAN maintains consistent performance across all orders. For the forward case and dual cascade, the performance of FEX-TFDM and FEX-VAE degrades at higher orders, with FEX-VAE showing a smaller decline than FEX-TFDM. Nevertheless, as shown in Figure~\ref{fig:discussion_density}, all three models including FEX-TFDM, FEX-SRAN, and FEX-VAE capture the underlying density accurately, suggesting that the degradation in higher-order moments does not significantly compromise the overall distributional recovery.

\begin{figure}[!htbp]
\begin{center}
    \includegraphics[width=0.9\linewidth]{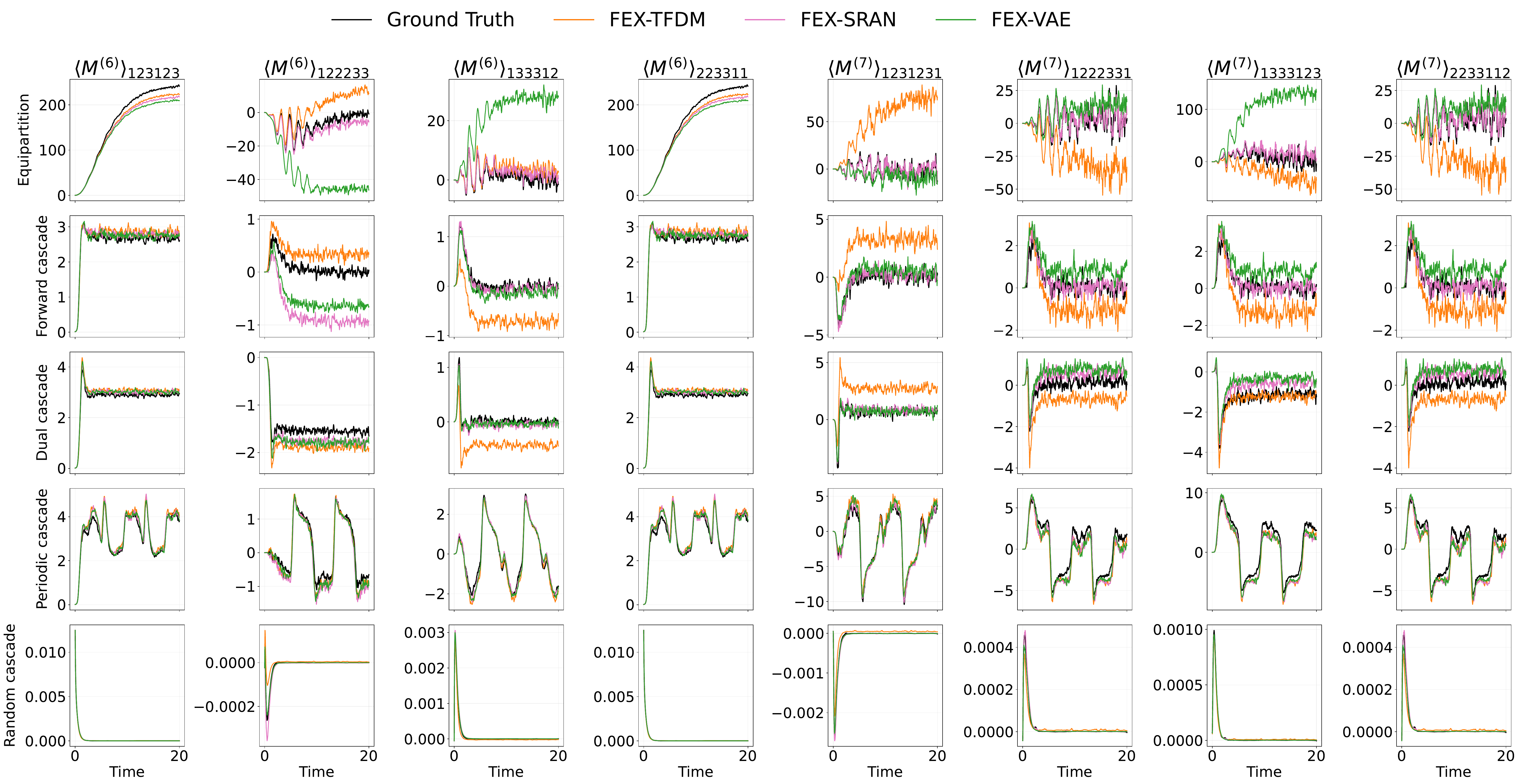}
\caption{Comparison of predicted sixth- and seventh-order moments across all regimes: equipartition, forward cascade, dual cascade, periodic cascade, and random cascade. 
Here, $M^{(k)}$ denotes moments of order $k$. For sixth-order moments, $M^{(6)}_{ijklmn} = \mathbb{E}[u_i u_j u_k u_l u_m u_n]$, and for seventh-order moments, $M^{(7)}_{ijkl mnp} = \mathbb{E}[u_i u_j u_k u_l u_m u_n u_p]$, where the indices specify the corresponding components.}
\label{fig:discussion_moment_equal_67}
\end{center}
\end{figure}

\begin{figure}[!htbp]
\begin{center}
    \includegraphics[width=1\linewidth]{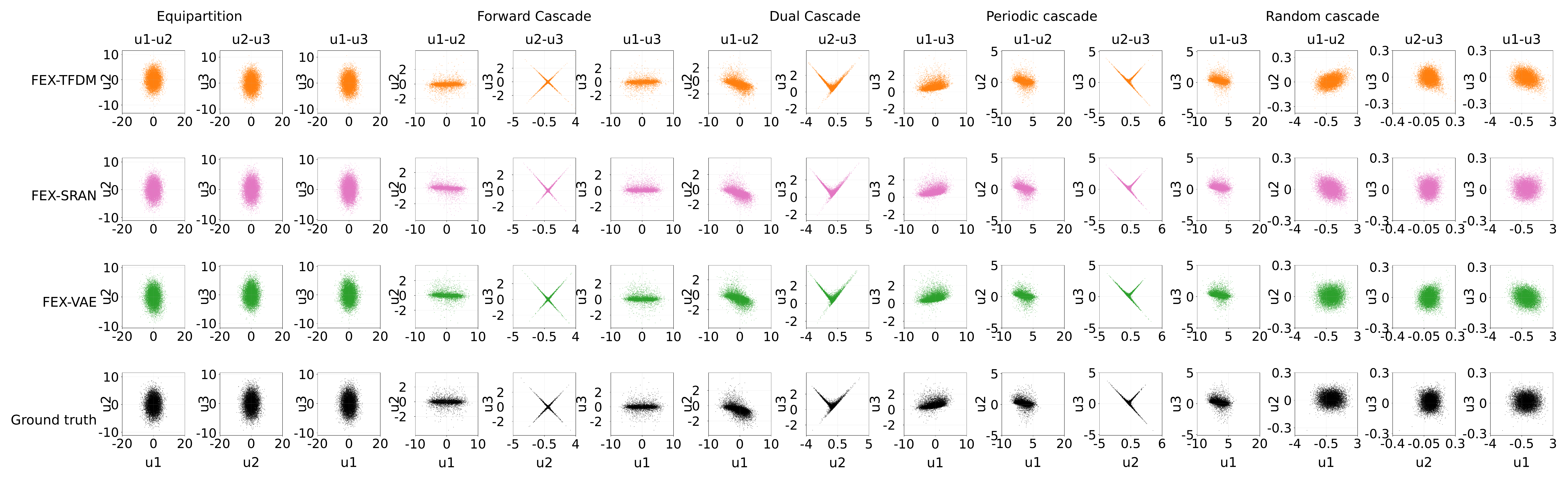}
\caption{Evolution of 2D projections of state distributions \(\mathbf{u} = (u_1,u_2,u_3)\) across all regimes. Results are obtained using the two-stage framework with FEX combined with TFDM, SRAN, and VAE. Each regime is displayed in three consecutive columns, and labels such as \(u_1-u_3\) indicate the corresponding 2D projection, here onto the 
\((u_1,u_3)\)-plane.}
\label{fig:discussion_density}
\end{center}
\end{figure}

\section*{Code Availability}

All numerical results presented in this work are fully reproducible using the publicly available repository at
\url{https://github.com/xxjan719/TurbulentFEX}.

\bibliographystyle{siam_style_file/siamplain}
\bibliography{main/main}

%% file: share_information/ex_shared.tex
\newsiamremark{remark}{Remark}
\newsiamremark{hypothesis}{Hypothesis}
\crefname{hypothesis}{Hypothesis}{Hypotheses}
\newsiamthm{claim}{Claim}
\newsiamthm{assum}{Assumption}
\newsiamthm{scheme}{Scheme}
\newsiamthm{lem}{Lemma}
\newsiamthm{thm}{Theorem}
\headers{Symbolic Discovery}{ X.~Xu,  D.~Qi, C.~Wang}

\title{
The Finite Expression Method for Turbulent Dynamics with High-Order Moment Recovery \thanks{C. W. was partially supported  by the National Science Foundation under award DMS-2206332.
D. Q. is partially supported by ONR grant N00014-24-1-2192 and NSF grant DMS-2407361.}}

\author{
Xingjian Xu\thanks{Department of Mathematics, University of Florida, Gainesvile, FL 32611 (\email{xingjianxu@ufl.edu}).}
\and 
Di Qi\thanks{Department of Mathematics, Purdue University, West Lafayette, IN 47907(\email{qidi@purdue.edu}).}
\and 
Chunmei Wang
\thanks{Department of Mathematics, University of Florida, Gainesvile, FL 32611 (\email{chunmei.wang@ufl.edu}).}
}



%% file: main/main.bib
@article{liang2025finite,
  title={Finite expression method for solving high-dimensional partial differential equations},
  author={Liang, Senwei and Yang, Haizhao},
  journal={Journal of Machine Learning Research},
  volume={26},
  number={138},
  pages={1--31},
  year={2025}
}

@article{du2025learning,
  title={Learning epidemiological dynamics via the finite expression method},
  author={Du, Jianda and Liang, Senwei and Wang, Chunmei},
  journal={Journal of machine learning for modeling and computing},
  volume={6},
  number={1},
  year={2025},
  publisher={Begel House Inc.}
}

@article{song2025finite,
  title={A finite expression method for solving high-dimensional committor problems},
  author={Song, Zezheng and Cameron, Maria K and Yang, Haizhao},
  journal={SIAM Journal on Scientific Computing},
  volume={47},
  number={1},
  pages={C1--C21},
  year={2025},
  publisher={SIAM}
}

@article{liang5327407identifying,
  title={Identifying Stochastic Dynamics Via Finite Expression Methods},
  author={LIANG, SENWEI and Wang, Chunmei and XU, XINGJIAN},
  journal={Available at SSRN 5327407}
}

@article{liu2025training,
  title={A training-free conditional diffusion model for learning stochastic dynamical systems},
  author={Liu, Yanfang and Chen, Yuan and Xiu, Dongbin and Zhang, Guannan},
  journal={SIAM Journal on Scientific Computing},
  volume={47},
  number={5},
  pages={C1144--C1171},
  year={2025},
  publisher={SIAM}
}

@article{gluhovsky1997interpretation,
  title={An interpretation of atmospheric low-order models},
  author={Gluhovsky, Alexander and Agee, Ernest},
  journal={Journal of the atmospheric sciences},
  volume={54},
  number={6},
  pages={768--773},
  year={1997}
}

@article{majda2002priori,
  title={A priori tests of a stochastic mode reduction strategy},
  author={Majda, Andrew and Timofeyev, Ilya and Vanden-Eijnden, Eric},
  journal={Physica D: Nonlinear Phenomena},
  volume={170},
  number={3-4},
  pages={206--252},
  year={2002},
  publisher={Elsevier}
}

@book{majda2016introduction,
  title={Introduction to turbulent dynamical systems in complex systems},
  author={Majda, Andrew J},
  year={2016},
  publisher={Springer}
}

@incollection{randall2007climate,
  title={Climate models and their evaluation},
  author={Randall, David A and Wood, Richard A and Bony, Sandrine and Colman, Robert and Fichefet, Thierry and Fyfe, John and Kattsov, Vladimir and Pitman, Andrew and Shukla, Jagadish and Srinivasan, Jayaraman and others},
  booktitle={Climate change 2007: The physical science basis. Contribution of Working Group I to the Fourth Assessment Report of the IPCC (FAR)},
  pages={589--662},
  year={2007},
  publisher={Cambridge University Press}
}

@article{loshchilov2016sgdr,
  title={Sgdr: Stochastic gradient descent with warm restarts},
  author={Loshchilov, Ilya and Hutter, Frank},
  journal={arXiv preprint arXiv:1608.03983},
  year={2016}
}

@article{brunton2016discovering,
  title={Discovering governing equations from data by sparse identification of nonlinear dynamical systems},
  author={Brunton, Steven L and Proctor, Joshua L and Kutz, J Nathan},
  journal={Proceedings of the national academy of sciences},
  volume={113},
  number={15},
  pages={3932--3937},
  year={2016},
  publisher={National Academy of Sciences}
}

@article{messenger2021weak,
  title={Weak SINDy for partial differential equations},
  author={Messenger, Daniel A and Bortz, David M},
  journal={Journal of Computational Physics},
  volume={443},
  pages={110525},
  year={2021},
  publisher={Elsevier}
}

@article{wanner2024higher,
  title={On higher order drift and diffusion estimates for stochastic SINDy},
  author={Wanner, Mathias and Mezi{\'c}, Igor},
  journal={SIAM Journal on Applied Dynamical Systems},
  volume={23},
  number={2},
  pages={1504--1539},
  year={2024},
  publisher={SIAM}
}

@article{lenfesty2025uncovering,
  title={Uncovering dynamical equations of stochastic decision models using data-driven SINDy algorithm},
  author={Lenfesty, Brendan and Bhattacharyya, Saugat and Wong-Lin, KongFatt},
  journal={Neural Computation},
  volume={37},
  number={3},
  pages={569--587},
  year={2025},
  publisher={MIT Press 255 Main Street, 9th Floor, Cambridge, Massachusetts 02142, USA~…}
}

@article{xu2024modeling,
  title={Modeling unknown stochastic dynamical system via autoencoder},
  author={Xu, Zhongshu and Chen, Yuan and Chen, Qifan and Xiu, Dongbin},
  journal={Journal of Machine Learning for Modeling and Computing},
  volume={5},
  number={3},
  year={2024},
  publisher={Begel House Inc.}
}

@article{wang2026error,
  title={Error estimates of a training-free diffusion model for high-dimensional sampling},
  author={Wang, Pengjun and Zhang, Zezhong and Yang, Minglei and Bao, Feng and Cao, Yanzhao and Zhang, Guannan},
  journal={arXiv preprint arXiv:2601.19740},
  year={2026}
}

@article{majda2018strategies,
  title={Strategies for reduced-order models for predicting the statistical responses and uncertainty quantification in complex turbulent dynamical systems},
  author={Majda, Andrew J and Qi, Di},
  journal={SIAM Review},
  volume={60},
  number={3},
  pages={491--549},
  year={2018},
  publisher={SIAM}
}

@article{qi2016low,
  title={Low-dimensional reduced-order models for statistical response and uncertainty quantification: Two-layer baroclinic turbulence},
  author={Qi, Di and Majda, Andrew J},
  journal={Journal of the Atmospheric Sciences},
  volume={73},
  number={12},
  pages={4609--4639},
  year={2016}
}

@phdthesis{qi2017strategies,
  title={Strategies for Reduced-Order Models in Uncertainty Quantification of Complex Turbulent Dynamical Systems},
  author={Qi, Di},
  year={2017},
  school={New York University}
}

@article{wang2025simulating,
  title={Simulating three-dimensional turbulence with physics-informed neural networks},
  author={Wang, Sifan and Sankaran, Shyam and Fan, Xiantao and Stinis, Panos and Perdikaris, Paris},
  journal={arXiv preprint arXiv:2507.08972},
  year={2025}
}

@article{lai2025h,
  title={H-FEX: A Symbolic Learning Method for Hamiltonian Systems},
  author={Lai, Jasen and Liang, Senwei and Wang, Chunmei},
  journal={arXiv preprint arXiv:2506.20607},
  year={2025}
}

@article{hardwick2025solving,
  title={Solving high-dimensional partial integral differential equations: The finite expression method},
  author={Hardwick, Gareth and Liang, Senwei and Yang, Haizhao},
  journal={Journal of Computational Physics},
  pages={114273},
  year={2025},
  publisher={Elsevier}
}

@article{huynh2025score,
  title={A Score-based Diffusion Model Approach for Adaptive Learning of Stochastic Partial Differential Equation Solutions},
  author={Huynh, Toan and Fajardo, Ruth Lopez and Zhang, Guannan and Ju, Lili and Bao, Feng},
  journal={arXiv preprint arXiv:2508.06834},
  year={2025}
}

@article{yang2025generative,
  title={Generative AI models for learning flow maps of stochastic dynamical systems in bounded domains},
  author={Yang, Minglei and Liu, Yanfang and Del-Castillo-Negrete, Diego and Cao, Yanzhao and Zhang, Guannan},
  journal={Journal of Computational Physics},
  pages={114434},
  year={2025},
  publisher={Elsevier}
}

@article{petersen2019deep,
  title={Deep symbolic regression: Recovering mathematical expressions from data via risk-seeking policy gradients},
  author={Petersen, Brenden K and Landajuela, Mikel and Mundhenk, T Nathan and Santiago, Claudio P and Kim, Soo K and Kim, Joanne T},
  journal={arXiv preprint arXiv:1912.04871},
  year={2019}
}

@article{goodfellow2014generative,
  title={Generative adversarial nets},
  author={Goodfellow, Ian J and Pouget-Abadie, Jean and Mirza, Mehdi and Xu, Bing and Warde-Farley, David and Ozair, Sherjil and Courville, Aaron and Bengio, Yoshua},
  journal={Advances in neural information processing systems},
  volume={27},
  year={2014}
}

@article{tzen2019neural,
  title={Neural stochastic differential equations: Deep latent gaussian models in the diffusion limit},
  author={Tzen, Belinda and Raginsky, Maxim},
  journal={arXiv preprint arXiv:1905.09883},
  year={2019}
}

@article{kingma2013auto,
  title={Auto-encoding variational bayes},
  author={Kingma, Diederik P and Welling, Max},
  journal={arXiv preprint arXiv:1312.6114},
  year={2013}
}

@article{gu2023stationary,
  title={Stationary density estimation of It{\^o} diffusions using deep learning},
  author={Gu, Yiqi and Harlim, John and Liang, Senwei and Yang, Haizhao},
  journal={SIAM Journal on Numerical Analysis},
  volume={61},
  number={1},
  pages={45--82},
  year={2023},
  publisher={SIAM}
}

@article{frishman2020learning,
  title={Learning force fields from stochastic trajectories},
  author={Frishman, Anna and Ronceray, Pierre},
  journal={Physical Review X},
  volume={10},
  number={2},
  pages={021009},
  year={2020},
  publisher={APS}
}

@book{majda2006nonlinear,
  title={Nonlinear dynamics and statistical theories for basic geophysical flows},
  author={Majda, Andrew and Wang, Xiaoming},
  year={2006},
  publisher={Cambridge University Press}
}

@book{vallis2017atmospheric,
  title={Atmospheric and oceanic fluid dynamics},
  author={Vallis, Geoffrey K},
  year={2017},
  publisher={Cambridge University Press}
}

@book{nicholson1983introduction,
  title={Introduction to plasma theory},
  author={Nicholson, Dwight Roy and Nicholson, Dwight R},
  volume={1},
  year={1983},
  publisher={Wiley New York}
}

@article{lorenz1969predictability,
  title={The predictability of a flow which possesses many scales of motion},
  author={Lorenz, Edward N},
  journal={Tellus},
  volume={21},
  number={3},
  pages={289--307},
  year={1969}
}

@article{majda2010quantifying,
  title={Quantifying uncertainty in climate change science through empirical information theory},
  author={Majda, Andrew J and Gershgorin, Boris},
  journal={Proceedings of the National Academy of Sciences},
  volume={107},
  number={34},
  pages={14958--14963},
  year={2010}
}

@article{ling2016reynolds,
  title={Reynolds averaged turbulence modelling using deep neural networks with embedded invariance},
  author={Ling, Julia and Kurzawski, Andrew and Templeton, Jeremy},
  journal={JFM},
  volume={807},
  pages={155--166},
  year={2016}
}

@article{duraisamy2019turbulence,
  title={Turbulence modeling in the age of data},
  author={Duraisamy, Karthik and Iaccarino, Gianluca and Xiao, Heng},
  journal={Annual Review of Fluid Mechanics},
  volume={51},
  pages={357--377},
  year={2019}
}

@article{pope2001turbulent,
  title={Turbulent flows},
  author={Pope, Stephen B},
  journal={Measurement Science and Technology},
  volume={12},
  number={11},
  pages={2020--2021},
  year={2001}
}

@article{rangan2009multiscale,
  title={Multiscale modeling of the primary visual cortex},
  author={Rangan, Aaditya V and Tao, Louis and Kovacic, Gregor and Cai, David},
  journal={IEEE Engineering in Medicine and Biology Magazine},
  volume={28},
  number={3},
  pages={19--24},
  year={2009},
  publisher={IEEE}
}
